\documentclass{article}

\usepackage{times}
\usepackage{fancyhdr}
\usepackage{fancybox}
\usepackage{changepage}
\usepackage{mathtools}
\usepackage{bbm}
\usepackage{color}
\usepackage{csquotes}
\usepackage{breqn}
\usepackage{amssymb}
\usepackage{amsmath}
\usepackage{amsthm}
\usepackage{hyperref}
\usepackage[round,comma]{natbib} 
\usepackage{appendix} 
\usepackage[preprint]{neurips_2019}

\usepackage[utf8]{inputenc} 
\usepackage[T1]{fontenc}    
\usepackage{hyperref}       
\usepackage{url}            
\usepackage{booktabs}       
\usepackage{amsfonts}       
\usepackage{nicefrac}       
\usepackage{microtype}      
\usepackage{float}

\usepackage{wrapfig}

\usepackage[compact]{titlesec}
\titlespacing{\section}{0pt}{*0}{*0}
\titlespacing{\subsection}{0pt}{*0}{*0}
\titlespacing{\subsubsection}{0pt}{*0}{*0}

\usepackage{algorithm} 
\usepackage{algpseudocode}
\renewcommand{\algorithmicrequire}{\textbf{Input:}}
\renewcommand{\algorithmicensure}{\textbf{Output:}}

\usepackage{graphicx}
\graphicspath{{.}{figs/}}
\usepackage{booktabs}

\usepackage{amsmath,amsfonts,amssymb}
\usepackage{mathletters} 
\usepackage{amsthm}

\newcommand{\beq}{\begin{equation}}
\newcommand{\eeq}{\end{equation}}
\newcommand{\baln}{\begin{align*}}
\newcommand{\ealn}{\end{align*}}
\newcommand{\beqn}{\begin{equation*}}
\newcommand{\eeqn}{\end{equation*}}
\newcommand{\beqa}{\begin{eqnarray}}
\newcommand{\eeqa}{\end{eqnarray}}
\newcommand{\beqan}{\begin{eqnarray*}}
\newcommand{\eeqan}{\end{eqnarray*}}
\newtheorem*{theorem*}{Theorem}
\newtheorem{proposition}{Proposition}
\newtheorem{theorem}{Theorem}
\newtheorem{lemma}{Lemma}
\newtheorem{corollary}{Corollary}
\newtheorem{definition}{Definition}
\newtheorem{assumption}{Assumption}
\newtheorem{remark}{Note}
\newcommand{\norm}[1]{\left\lVert#1\right\rVert}
\newcommand{\expect}[1]{\mathbb{E}\left[{#1}\right]}

\newcommand{\argmax}{\mathop{\mathrm{argmax}}}

\newcommand{\argmin}{\mathop{\mathrm{argmin}}}

\newenvironment{mylemma}[1]{

		\begin{minipage}{0.91\textwidth}
			\vspace{2mm}	
			\begin{lemma}[#1]
	}{
\end{lemma}
\vspace{0mm}
\end{minipage}

	}

\newenvironment{mycorollary}[1]{
		
		\begin{minipage}{0.91\textwidth}
			\vspace{2mm}	
			\begin{corollary}[#1]
	}{
\end{corollary}
\vspace{0mm}
\end{minipage}

	}

\newenvironment{mytheorem}[1]{

		\begin{minipage}{0.91\textwidth}
			\vspace{2mm}	
			\begin{theorem}[#1]
	}{
\end{theorem}
\vspace{0mm}
\end{minipage}

	}

\newenvironment{myassumption}[1]{

		\begin{minipage}{0.91\textwidth}
			\vspace{2mm}	
			\begin{assumption}[#1]
	}{
\end{assumption}
\vspace{0mm}
\end{minipage}

	}

\title{Towards Optimal and Efficient Best Arm Identification in Linear Bandits}	
	
\author{%
Mohammadi Zaki\\
 Electrical Communication Engineering\\
 Indian Institute of Science\\
 Bangalore.\\
\texttt{mohammadi@iisc.ac.in}
\And
Avinash Mohan \\
 Faculty of Electrical Engineering\\
 Israel Institute of Technology (Technion),\\
 Haifa. \\
 \texttt{avinashm.1214@gmail.com}
\And
Aditya Gopalan \\
 Electrical Communication Engineering\\
    Indian Institute of Science\\
    Bangalore. \\
\texttt{aditya@iisc.ac.in}
}
\begin{document}
\maketitle
\thispagestyle{empty}
%
\begin{center}
\textbf{\textit{Abstract.}}	\\
\end{center}
We give a new algorithm for best arm identification in linearly parameterised bandits in the fixed confidence setting. The algorithm generalises the well-known LUCB algorithm of \citet{Kalyanakrishnan2012PACSS} by playing an arm which minimises a suitable notion of geometric overlap of the statistical confidence set for the unknown parameter, and is fully adaptive and computationally efficient as compared to several state-of-the methods. We theoretically analyse the sample complexity of the algorithm for problems with two and three arms, showing optimality in many cases. Numerical results indicate favourable performance over other algorithms with which we compare.



\section{Introduction}

This paper is concerned with the  problem of optimising an unknown linear function over a finite domain when given the ability to sequentially test and observe noisy function values at domain points of our choice. In the language of online learning this is the problem of best arm identification in linearly parameterised bandits; in classical statistics it is essentially the problem of adaptive, sequential composite hypothesis testing where each hypothesis corresponds to one domain point being optimal. From the point of view of causal inference, it can be interpreted as the problem of learning the best (i.e., most rewarding) intervention, from among a set of parameterised interventions available at hand with respect to an observed variable \citep{lattimore2016causal, sen2016contextual}, with the key difference in this work being that the causal response of the variable to an intervention is modelled as being linear in the intervention value. The linear structure endows the model with complex but exploitable structure, in that it makes possible inference about the utility (function value) of an intervention (bandit arm) by using observations from other, correlated interventions, akin to what happens in standard (batch) prediction with linear regression.

In the linear bandit setting,  each arm or action is associated with a fixed known feature vector $x\in \mathbb{R}^d$ and the expected reward obtained by choosing to pull arm with feature vector $x$ is $x^T\theta$ where $\theta$ is a fixed but unknown vector. We specifically consider the probably-approximately-correct (PAC) objective of the learner (agent) declaring a guess for the identity of the optimal arm, after it has made an internally determined (and potentially random) number of sequential plays of arms, which is required to be correct with at least a given probability $(1-\delta)$ -- the {\em fixed confidence} best arm identification goal \citep{evendar}. In this regard, our focus is on both the statistical and computational efficiency of adaptive arm-sampling strategies, i.e., designing strategies (a) whose number of plays is as close to the quantifiable information-theoretic limit on sample complexity across all strategies, and (b) which can determine in a computationally lightweight manner the next arm to play at each adaptive round. 

Broadly, there are two different approaches towards solving such pure exploration problems: (i) uniform sampling-elimination based and (ii) adaptive sampling-UCB based. The algorithm of \cite{Tao_et-al-ICML-2018} is based on the former approach while those of \cite{soare} and \cite{XuAISTATS} are based on the latter, adaptive-sampling idea, but both have sample complexity guarantees that do not depend finely on the problem instance and linear structure, and thus are worst-case optimal at best.

\par The $\mathcal{X} \mathcal{Y}$-static algorithm of \cite{soare} is a static algorithm which fixes the schedule of arm plays before collecting any observations. Hence, it is not able to adapt towards pulling arms which may be more ``informative" for identifying the best arm for the given instance, and can consequently only be worst-case optimal. The LinGapE algorithm  \cite{XuAISTATS} is a fully adaptive algorithm which  performs well experimentally. However it requires to solve  $\Omega(K^2)$ optimization problems at start, one for every pair of arms ($K$ denotes the number of arms), which is computationally inefficient for large values of $K$. Finally, the $\mathcal{Y}$-ElimTil-$p$ algorithm of \cite{Tao_et-al-ICML-2018} is an elimination based algorithm, and though its sample complexity scales only linearly with the dimension $d$, it requires to sample $\Omega(d)$ arms in each round which is already far from optimal even for the case of the canonical (unstructured) multi-armed bandit (MAB) having the standard basis vectors as the arms. This can again only be worst-case optimal. A summary of the sample complexity bounds of the above algorithms appears in Table \ref{table:comparison}.

A very recent departure from this worst-case sample complexity dependence is the work of \citet{jamieson-etal19transductive-linear-bandits}, that shows a provably instance-optimal best arm identification algorithm for linear (and more generally transductive) bandits. However, implementation of this algorithm requires computing an arguably costly rounding procedure to determine (in phases) a schedule of arms to play, as well as solving a minimax optimization problem which may be computationally very expensive\footnote{In fact, these aspects of their algorithm have prevented us from successfully implementing and testing it.}. 

\textbf{Our contributions and organization.} In contrast with existing work, we aim to take a qualitatively different route towards the design of linear best arm identification, by drawing inspiration from the upper confidence bound principle, which is known to give sample-optimal performance for canonical MAB \citep{Kalyanakrishnan2012PACSS, jamieson2014lil}. In this conceptually simple and elegant approach, in each decision round the learner constructs a statistically plausible confidence set for the underlying bandit instance (the weight vector $\theta$ in our linear setting) based on past observations, and then plays the arm that best appears to reduce the uncertainty about the optimal linear arm. 

We generalise the \textit{Lower Upper Confidence Bound} (LUCB) algorithm of \citet{Kalyanakrishnan2012PACSS} to the linear bandit setting. %
To achieve this, we introduce a new geometric ``maximum overlap'' principle as a basis for the learner to identify which arm is most informative to play at any given round. This results in a fully data-dependent arm selection strategy which we call \textit{Generalized-LUCB} (GLUCB) (Section \ref{sec:TheGlucbAlgorithm}). We then proceed to rigorously analyse the sample complexity of GLUCB for certain specialized (yet instructive) cases in Section \ref{upperbound},  and finally compare its empirical performance with other state-of-the-art methods in Section \ref{Experiments}.
\par As a comment on the execution times as compared to the other algorithms proposed for this problem, our proposed algorithm GLUCB, improves significantly on the time complexity over LinGapE \citep{XuAISTATS} and $\mathcal{Y}$-ElimTil-$p$ \citep{Tao_et-al-ICML-2018} as it does not require solving any offline optimization problems. 
\section{Problem Statement and Notation}\label{Problem}

We study the problem of best arm identification in linear multi-arm bandits (LMABs) with the arm set $\mathcal{X}\equiv \{x_1,x_2,\ldots,x_K \}$, where $K$ is finite but possibly large. We will interchangeably use $\mathcal{X}$ and the set $[K]\equiv\{1,2,\ldots,K\}$, whenever the context is clear. Each arm $x_a$ is a vector in $\mathbb{R}^d$. The quantity $d$ will, henceforth, be called the {\em ambient dimension.}  At every round $t=1,2,\ldots$ the agent chooses an arm $x_t \in  \mathcal{X} $, and receives a reward $y(x_t)={\theta^*}^Tx_t + \epsilon_t$, where $\theta^*$ is assumed to be a fixed but unknown vector and $\epsilon_t$ is zero-mean noise assumed to be conditionally $R$-subGaussian, i.e., $\forall \lambda\in \mathbb{R}$,
$ \expect{e^{\lambda \epsilon_t}|x_{a_1},x_{a_2},\ldots,x_{a_{t-1}}, \epsilon_1,\epsilon_2,\ldots,\epsilon_{t-1}} \leq \exp\Big(\frac{\lambda^2R^2}{2} \Big).$
  Let $a^*=\argmax\limits_{a\in [K]} {\theta^*}^Tx_a.$ The goal of the agent is, given an error probability $\delta$, to identify $ a^* $ with probability $\geq 1-\delta$, by pulling as few arms as possible (in literature, this is  known as the \enquote{fixed-$\delta$} regime \cite{Kaufman16}). Henceforth, we will call this the {\color{blue}LMAB} (linear multi armed bandit) problem. When restricted to the case where $K=d$ and $\mathcal{X}$ is the standard ordered basis $\{{\mathbf{e}}_1,\cdots,{\mathbf{e}}_d\},$ the problem reduces to the Standard LMAB (SMAB) problem studied, for instance, in \cite{Kalyanakrishnan2012PACSS}.
  \par In the rest of the paper, we will assume that $\norm{x_k}_2\leq 1, \forall x_k\in \mathcal{X}$ and that the agent has information of some upper bound on $||\theta^*||_2,$ say, $S$. Let $A$ be a positive definite matrix, then we denote by  $\norm{x}_{A}:=\sqrt{x^TAx}$, the matrix norm induced by $A$.
  Let for any $i\in [K], i\neq a^*$, $\Delta_{i}:={\theta^*}^T(x_{a^*}-x_{i})$ be the \textit{gap} between the largest expected reward and the expected reward for arm $x_i$. Denote by $\Delta_{min}:=\min\limits_{\substack{i\in\mathcal{X}\setminus\{a^*\}}} \Delta_i$, the smallest reward gap.

  \begin{table}[tbh]
  
    \caption{Comparison of Sample complexities achieved by various algorithms for the LMAB problem in the literature. Note that $K$ is the number of arms, $d$ is the ambient dimension, $\delta$ is the PAC guarantee parameter and $\Delta_{min}$ is the minimum reward gap.}
    \label{table:comparison}
    \centering
    \begin{tabular}{c|c}
      \toprule
      {\bf Algorithm} & {\bf Sample Complexity}\\
      \hline
      {\color{blue}$\mathcal{X} \mathcal{Y}$-static} \citep{soare} & $O\Big(\frac{d}{\Delta_{min}}(\ln\frac{1}{\delta} + \ln K + \ln \frac{1}{\Delta_{min}}) +d^2\Big)$\\
      {\color{blue}LinGapE}\footnote{Here $H_0$ is a complicated term defined in terms of a solution to an offline optimization problem in \cite{XuAISTATS}.} \citep{XuAISTATS} & $O\Big(dH_0\log\Big(dH_0\log\frac{1}{\delta}\Big) \Big)$\\
      {\color{blue}$\mathcal{Y}$-ElimTil-$p$} \citep{Tao_et-al-ICML-2018} & $O\Big(\frac{d}{\Delta_{min}}(\ln\frac{1}{\delta} + \ln K + \ln\ln \frac{1}{\Delta_{min}})\Big)$\\
     {\color{blue}RAGE} \citep{jamieson-etal19transductive-linear-bandits} & Instance-dependent lower bound (upto log factors)  \\
      \bottomrule
    \end{tabular}
  
  \end{table}
\section{The GLUCB Algorithm}\label{sec:TheGlucbAlgorithm}
This section is organized as follows. We begin with a description of the ingredients required to construct GLUCB, including \enquote{MaxOverlap.} Thereafter, we formally describe the GLUCB algorithm. Finally, we show how GLUCB is a generalization of LUCB. 

To begin with, note that any algorithm for the best arm identification problem requires the following ingredients: 
\begin{enumerate}\label{ingredientsForLMABAlgorithms}
\item a {\color{blue}stopping rule}: which decides when the agent must stop sampling arms, and is a function of past observations, arms chosen and rewards only, 
\item a {\color{blue}sampling rule}: which determines, based on the arms played and rewards observed hitherto, which arm to pull next (clearly, this rule is invoked only if the stopping rule decides not to stop); and
\item a {\color{blue}recommendation rule}: which, when the stopping rule decides to stop, chooses the index of the arm that is to be reported as the best.
\end{enumerate}
 Each of these steps will now be developed in detail and combined to give the full GLUCB algorithm. Towards this we first introduce some technical desiderata.

 Let $(x_1,x_2,\ldots,x_t)\in \mathcal{X}$ be a sequence of arms played  until time $t$ by any adaptive strategy (i.e., a strategy which chooses to play an arm depending on the past arm pulls and their corresponding observations) and let $(y_1,y_2,\ldots,y_t)$ be the received rewards. The (regularized) \emph{least squares estimate} $\theta_t$ of $\theta^*$ at time $t$ is given by
 \begin{eqnarray*}
 \theta_t &:=& V_t^{-1}b_t,\text{ where }\\
 V_t &:=& \lambda I + \sum\limits_{s=1}^{t}x_sx_s^T, \text{ and } b_t := \sum\limits_{s=1}^{t}x_sy_s.
 \end{eqnarray*}
 By standard results on least squares confidence sets for adaptive sampling \citep{NIPS2011_4417}, it can be shown that 
  with high probability, $\theta^*$ lies in the \emph{confidence ellipsoid}\footnote{recall that noise is assumed to be $R$-sub Gaussian}
  \begin{eqnarray*}
      \mathcal{E}_t\left(\theta_t,V_t\right) := \left\lbrace \theta\in\mathbb{R}^d\bigg| \norm{\theta_t-\theta}_{V_t} \leq \beta_t \right\rbrace,\text{ where }
  \beta_t := R\sqrt{d\log\frac{t}{\delta}}.
  \end{eqnarray*}
  Notice that the ellipsoid is time-indexed, since, as more arms chosen, the estimate changes and so does the ellipsoid. In the sequel, we sometimes denote $\mathcal{E}_t\left(\theta_t,V_t\right)$ by $\mathcal{C}_t$ in the interest of space.
    \par We also define a set $\text{HalfSpace}(i,j):=\{z\in \mathbb{R}^d: z^T(x_i-x_j)\geq 0 \}$, where $x_i$ and $x_j$ are vectors in $\mathcal{X}$. Next, for any $x_k\in \mathcal{X}$, define $R(x_k):=\{\theta\in \mathbb{R}^d: \theta^Tx_k \geq \theta^Tx_j, \forall j\neq k \}$, as the cone of parameters $\theta$ such that, if $\theta^*\in R(x_k)$, then $x_k$ is the optimal arm. Clearly $\{R(x_k), x_k\in \mathcal{X}\}$ partition the entire parameter space $\mathbb{R}^d$, modulo the degenerate regions where more than one arm is optimal. 
    Furthermore, let $h_t:=\argmax_{i\in\{1,\cdots,K\}}\theta_t^Tx_i,$ be the index of the arm that currently appears to be the best. 
\subsection{Ingredients of the GLUCB algorithm}\label{subsec:Desideratum}
 Following the intuition in \citep[Sec.~3]{soare}, we observe that a good choice for a stopping rule could be to stop the algorithm when the confidence ellipsoid $\mathcal{C}_t$ is completely contained within one of the $K$ cones $R(x_k),1\leq k\leq K$.  Therefore, we wish to design an algorithm which minimizes the overlap of the current confidence ellipsoid with every cone which currently seems to be suboptimal, i.e., all the cones other than the current home cone of $\theta_t$. That way, the algorithm can quickly insert $C_t$ completely into one of the cones, and because $C_t$ contains $\theta^*$ with high probability, so does $R\left(x_{h_t}\right)$ since, now, $\mathcal{C}_t\subset R\left(x_{h_t}\right).$ This also means that upon stopping, $h_{t}$ will be the arm recommended.
 \begin{definition}[\color{blue}MaxOverlap]
   The MaxOverlap of set $A$ on set $B$ is defined  to be the maximum  distance of set $A$ from the boundary of another set $B$. 
   \begin{align*}
   MaxOverlap(A;B) :=
   \begin{cases}
   \max\limits_{e_A\in A\cap \bar{B}} \min\limits_{e_B \in \partial(\bar{B})\cap A} \norm{e_A-e_B}_2, & if \quad A\cap \bar{B} \neq \emptyset,\\
   0, & \text{otherwise.}
   \end{cases}
   \end{align*}
   \end{definition}
    Here $\bar{B}$ denotes the closure of the set $B$ and $\partial(\bar{B})$ its topological boundary (\cite{rudin64principles-mathematical-analysis}).
   Hence, at time $t+1$, our algorithm GLUCB is defined as sampling the  arm  
   \begin{align*}
   a_{t+1}:= \argmin\limits_{a=1}^K MaxOverlap\Bigg(\mathcal{E}(\theta_t, V_t+x_ax_a^T); R_{h_t}^c\Bigg).
   \end{align*}
   
   The following result shows that the MaxOverlap-based arm sampling rule reduces to a concrete prescription for the linear bandit setting. (Due to space constraints, proof details are omitted and can be found in the Appendix.)  
  \begin{proposition} 
  At time step $t$, define the arm
  \begin{equation*}
      l_t :=\argmax\limits_{i\neq h_t, i\in [K]} MaxOverlap\Bigg(\mathcal{E}(\theta_t, V_t); \text{HalfSpace}(i,h_t) \Bigg).
  \end{equation*}
 Then, 
$a_{t+1}\in\argmax\limits_{a=1}^K \frac{\abs{x_a^TV_t^{-1}(x_{l_t}-x_{h_t})}}{\sqrt{1+ \norm{x_a}^2_{V_t^{-1}}}}.$
  \end{proposition}\label{Prop:MaxOverlapArm}
 We are now ready to describe the ingredients of GLUCB. At every time step $t,$ define the \enquote{Advantage} of arm $a\in\mathcal{X}$ as 
 \begin{equation*}
     \text{Advantage}(a) := \max\limits_{\theta \in \mathcal{E}_t\left(\theta_t,V_t\right)}(\theta^Tx_a-\theta^Tx_{h_t}).
 \end{equation*}
 
\textbf{Stopping rule:} The algorithm stops when the \enquote{Advantage} defined above becomes non-positive for every arm other than the current best arm $h_t.$ \\
\textbf{Sampling rule:} Play the arm which minimizes the current max Advantage:
$a_{t+1} \in \argmax\limits_{a \in[K]} \abs{\frac{x_a^TV_t^{-1}(x_{h_*^t}-x_{l_*^t})}{\sqrt{1+x_a^TV_t^{-1}x_a}}}. $ In case of a tie, the agent selects an arm uniformly randomly.\\
\textbf{Recommendation rule:} Once the algorithm stops, the current best arm $h_*^t$ is recommended as the guess for the best arm.

\begin{algorithm}[htbp]
\small
 	\caption{{\bf GLUCB (Generalized Lower and Upper Confidence Bounds)}} \label{alg:Main}
 	\begin{algorithmic}[1]
	\State {\bf Input:} $\delta, R, S. $
    \State {\bf Initialize:} $V_0\leftarrow\lambda I$, ${\theta}_0\leftarrow\underline{0}, \beta_0 \leftarrow 1, b_0\leftarrow 0, t\leftarrow 0, STOP \leftarrow 0.$
 	\While {STOP != 1}
 	\State $\mathcal{C}_t \leftarrow \{\theta \in \mathbb{R}^d: \norm{\theta -\theta_t}_{V_t} \leq \beta_t \}$ \Comment{Form the high confidence ellipsoid}
 \State $ 	h_t\leftarrow\argmax\limits_{a\in [K]}\theta_t^Tx_a. $\Comment{Current best arm}
 \State $\forall a \in [K]\backslash\{h_t \}, \text{Advantage}(a) \leftarrow \max\limits_{\theta \in \mathcal{E}_t}(\theta^Tx_a-\theta^Tx_{h_t})$
  \If {$\max\limits_{a\in [K]}\text{Advantage}(a) < 0$} STOP $\leftarrow$ 1.\Comment{Stopping criterion}
 \Else  
 \State $l_t \leftarrow \argmax\limits_{a\in [K]}{\text{Advantage}(a)}$\Comment{``Closest" arm}
 \State $c_{t+1} \leftarrow \argmax\limits_{a \in [K]} \abs{\frac{x_a^TV_t^{-1}(x_{h_t}-x_{l_t})}{\sqrt{1+x_a^TV_t^{-1}x_a}}}.$
 \State Play $a_{t+1} \sim Unif\{ c_{t+1}\}.$
 \State Receive $y_t$.
 \State $V_{t+1} \leftarrow V_t+ x_{a_{t+1}}{x_{a_{t+1}}}^T.$
 \State $b_{t+1}\leftarrow b_t + y_tx_{a_{t+1}}.$
 \State $\theta_{t+1} \leftarrow V_{t+1}^{-1}b_{t+1}.$
 \State $\beta_{t+1} \leftarrow R\sqrt{2\log \frac{\det(V_{t+1})^{\frac{1}{2}}\det (\lambda I)^{-\frac{1}{2}}}{\delta}}+ \lambda^{\frac{1}{2}}S.$
 \State $t\leftarrow t+1.$
\EndIf
 \EndWhile
 \State \Return $h_*^t.$\Comment{Output the current best arm}
 
 	\end{algorithmic} 
 \end{algorithm}
Note that the GLUCB algorithm (Algorithm \ref{alg:Main}) reduces to the well-known \textit{LUCB} algorithm of \citet{Kalyanakrishnan2012PACSS} for the SMAB problem. If we consider the case when $d=K$, and the arms being the standard basis $\equiv \{e_1,e_2,\ldots,e_K \}$, we see that the arm $a_2$ in algorithm 1 corresponds to $l_t$ which is what LUCB would suggest. Indeed, when $j\neq l_t,h_t$, $\frac{\abs{x_a^TV_t^{-1}(x_{l_t}-x_{h_t})}}{\sqrt{1+x_a^TV_t^{-1}x_a}}=0$. However, with $j = l_t$ we obtain, $\frac{\abs{x_a^TV_t^{-1}(x_{l_t}-x_{h_t})}}{\sqrt{1+x_a^TV_t^{-1}x_a}}>0$. Also, it is easy to check the stopping criterion also reduces to that of LUCB. Hence, Algorithm \ref{alg:Main}, when applied to the unstructured case, plays the current best and the closest arm simultaneously every time till the algorithm stops.

 We now provide some preliminary theoretical results regarding the sample complexity performance of GLUCB.

\section{Analysis of GLUCB}\label{upperbound}
The following result proves the correctness of G-LUCB.
  \begin{theorem}
  Let ${\phi}:\mathcal{H}\to \left[K\right]$ be any arbitrary sampling strategy. Algorithm \ref{alg:Main} returns the optimal arm upon stopping with probability at least $ 1-\delta.$ 
  \end{theorem}
   We will now analyze the sample complexity of GLUCB when $K=2$ and $K=3$. For this we first present the following useful result on the convexity of a certain norm-based function on the probability simplex.
\begin{lemma}\label{lemma:convexity of y'V-1y}
For any $y\in \mathbb{R}^d$, and $A=\sum\limits_{i=1}^{K}\lambda_ix_ix_i^T$, where $\lambda_i\geq0$ and $\sum\limits_{i=1}^{K}\lambda_i=1$, the function $\lambda\mapsto y^TA^{-1}y$ is convex in $\lambda \in \mathcal{P}_K.$ 

\end{lemma}

\subsection{Analysis of GLUCB for Linear MAB with $K=2$ arms}\label{generaltwoarmcase}
Let $K=2$ for which the arm set is $\mathcal{X}\equiv \{x_1,x_2 \}$. Let $\norm{x_k}=1, k=1,2$ and $x_1^Tx_2=1-\rho,$ with $\rho>0$. For this simple case, it is clear that the set $\{h_*^t,l_*^t\}\equiv \{1,2 \}.$ We aim to analyze the sample complexity of GLUCB by tracking the possible sample paths of $(a_t)_{t\geq1}$ (which turns out to be tractable in this setting). Let $n_k(t), k=1,2$ be the number of times arm $k$ has been pulled till time $t$. Then, playing GLUCB guarantees that,
\begin{theorem}\label{2armupperbound}
In any round $t$,
$\Big\lfloor \frac{t}{2}\Big\rfloor \leq n_k(t) \leq \Big\lfloor \frac{t}{2}\Big\rfloor +1, \forall k=1,2.$

\end{theorem}

The proof of the result relies on the following observations.
\begin{lemma}\label{lemma:2arm1stlemma}
Whenever $n_1(t)=n_2(t),$
\begin{itemize}
\item[(i)] $x_1^TA^{-1}x_1=x_2^TA^{-1}x_2$, and
\item[(ii)] there is a tie.
\end{itemize}
\end{lemma}
The lemma tells us that for any $t\geq0$, if $n_1(t)=n_2(t)$ then there is a tie. Next, we show that whenever an arm (say arm 1 w.l.o.g) is played for $(n+1)$ times while arm 2 for $(n)$ times we will be forced to play arm 2.

\begin{lemma}\label{lemma:2arm2ndlemma}
Let $A=\lambda I + (n+1)x_1x_1^T +nx_2x_2^T.$ Then,
$ \frac{(x_1^TA^{-1}(x_1-x_2))^2}{1+x_1A^{-1}x_1} \leq \frac{(x_2^TA^{-1}(x_1-x_2))^2}{1+x_2A^{-1}x_2}.$
\end{lemma} 
Infact, for the two arm case, the sample complexity of GLUCB is optimal. The following `potential function'-based result formally establishes this fact.\\
Let us consider two algorithms $\mathcal{A}_1$ and $\mathcal{A}_2$, where $\mathcal{A}_1$ is GLUCB and $\mathcal{A}_2$ is any other algorithm. Let $\Phi^{\mathcal{A}_k}(t):={\norm{x_1-x_2}^2_{V_{\mathcal{A}_k}^{-1}}}$ for $k=1,2.$ We can now state 
\begin{theorem}\label{optimality of glucb for 2 arm}
$\forall t>0$, $\Phi^{\mathcal{A}_1}(t) \leq \Phi^{\mathcal{A}_2}(t)$.
\end{theorem}


On the other hand, a detailed analysis of the information-theoretic lower bound on sample complexity, e.g., \cite{Kaufman16} or \cite{jamieson-etal19transductive-linear-bandits}, yields the following result where $w^*$ is the optimal vector of arm frequencies in the min-max optimisation problem of the lower bound (termed $\lambda$ in \cite{jamieson-etal19transductive-linear-bandits}).

\begin{lemma}\label{lemma:lowerbound2armcase}[Lower bound for $K=2$].
For $K=2$, with any two arms $x_1$ and $x_2$ such $\norm{x_k}_2=1$, $w^*=\Big(\frac{1}{2}, \frac{1}{2}\Big).$
\end{lemma}
By combining theorem \ref{optimality of glucb for 2 arm} and lemma \ref{lemma:lowerbound2armcase} We have the following corollary.
\begin{corollary}
	The (expected) sample complexity of GLUCB for the  2-arm setting is at most $\beta_t^2H_G+1$, where $H_G = {\log\frac{1}{2.4\delta}} \min\limits_{\underline{w}\in \mathcal{P}_2}\max\limits_{a\neq a^*}\frac{\norm{x_a-x_{a^*}}^2_{W^{-1}}}{\Delta_{a}^2}$, where $\mathcal{P}_2$ is the probability simplex over 2 arms and $a^*$ represents the optimal arm.
\end{corollary}
\begin{remark}
The quantity $H_G$ is the usual information theoretic lower bound on best arm identification sample complexity \citep{jamieson-etal19transductive-linear-bandits}, with the sample complexity of GLUCB being only $O(d)$ away from it; the extra $d$ factor arises because of weaker concentration bounds for adaptive strategies.
\end{remark}
%

\subsection{Linear MAB with $K=3$ arms}\label{three-armupperbound}
This section deals with a representative example of the linear bandit. Let $K=3$, $d=2$ and the arm set $\mathcal{X}=\{e_1,e_2,(\cos(\omega), \sin(\omega))^T \}, 0<\omega<2\pi.$ Let $\theta^*=e_1$. This setup is particularly interesting when $\omega$ is close to 0. An algorithm which is optimal for standard MAB will quickly discard arm 2, and would continue to sample arms 1 and 3 until stopping. However, this is not the optimal strategy, since pulling arm 2 gives valuable information about $\theta^*$. We will see that this is what GLUCB does. 
\par As compared to the algorithms designed for best arm identification standard MAB, GLUCB identifies the structure (\textit{if any}) present in the arms and tries to exploit it. For the particular case just described, Arm 3 is \textit{always} dominated by the other two arms, i.e., we will see in the key Lemma \ref{lemma:key3armcase} that in order to minimize the uncertainity in any direction $x_i-x_j, i,j\in \{1,2,3\}, i\neq j$, the reduction obtained by pulling Arm 3 is always dominated by that for some other arm.

\begin{theorem}\label{thm:3armsamplecomplexity}
If GLUCB run on the above problem instance and $\tau$ is its stopping time, then,
\begin{equation*}
\begin{split}
\mathbb{P} & \Bigg\{\Bigg(\tau\leq \frac{4{\beta_t}^2}{\Delta_{min}^2}{\sin^2(\omega)} + \frac{4\beta_t^2}{\Delta_{min}}\sin(\omega) \; + \\&\max\left\{4\beta_t^2\left(\frac{\sin(\omega)}{\Delta_{min}} + 1 \right),\frac{4\beta_t^2}{\mu_3^2}\left(\frac{\sin(\omega)}{\Delta_{min}}{\cos(\omega)} + 1-\sin(\omega)\right) \right\}\Bigg) \land \left\{h_*^{\tau}==a^*\right\}\Bigg\} \geq 1-\delta.
\end{split}
\end{equation*}
\end{theorem}

The proof of this result relies on the following key lemma.

\begin{lemma}\label{lemma:key3armcase}
If $0<\omega<\pi/2$, then Arm 3 is never played.
\end{lemma}


Next, the following lemma shows an upper-bound on the time taken by GLUCB to discard arm 2 as $h_*^t$ or $l_*^t$. We show this by bounding the number of samples required by GLUCB such that $\text{Advantage}(2)<0$.
\begin{lemma}\label{lemma:3arm1stlemma}
With probability $\geq 1-\delta$, for all  $t>t_1$, where 
 $t_1:=\max\left\{4\beta_t^2\left(\frac{\sin(\omega)}{\Delta_{min}} + 1 \right),\frac{4\beta_t^2}{\mu_3^2}\left(\frac{\sin(\omega)}{\Delta_{min}}{\cos(\omega)} + 1-\sin(\omega)\right) \right\}$, Arm 3 $\notin \{h_*^t,l_*^t \}.$
 \end{lemma}

Finally we bound the number of samples needed by GLUCB to stop once the set $ \{h_*^t,l_*^t\} \equiv \{1,3\} $ has frozen.

\begin{lemma}\label{lemma:3arm2ndlemma}
The number of samples needed for G-LUCB to stop once in steady state is upper bounded by $ \frac{4{\beta_t}^2}{\Delta_{min}^2}{\sin^2(\omega)} + \frac{4\beta_t^2}{\Delta_{min}}\sin(\omega). $
\end{lemma}


\begin{remark}
The term inside the max in the theorem statement, is small and can be absorbed into the leading terms. Hence, the sample complexity of GLUCB can be written as $O\left( \frac{{\beta_t}^2}{\Delta_{min}^2}{\sin^2(\omega)} + \frac{\beta_t^2}{\Delta_{min}}\sin(\omega) \right)$.
\end{remark}
\begin{remark}
	A crucial observation here is that the geometry of the problem enters the sample complexity (in terms of $\sin\omega$), which. since $\omega$ is small, reduces the sample complexity compared to that of a standard MAB algorithm running on the instance.
\end{remark}
We will now see that this is indeed the optimal strategy.

\subsubsection{Lower-bound for the three-arm case}
%

By \cite[Theorem 1]{jamieson-etal19transductive-linear-bandits}, the expected sample complexity of any PAC best arm identification algorithm for LMAB is lower bounded as:
\[\frac{\expect{\tau}}{\log\frac{1}{2.4\delta}} \geq{\min\limits_{w\in \mathcal{P}_3}\max\limits_{a'\neq1} \frac{\norm{x_{a'}-x_1}^2_{W^{-1}}}{\Delta_{a'}^2}}\geq \min\limits_{w\in \mathcal{P}_3}\frac{\norm{x_3-x_1}_{W^{-1}}^2}{\Delta_{min}^2},\]
where $W:=\sum\limits_{a=1}^{3}w_ax_ax_a^T$. 
By solving the above optimization problem, %
we have
\begin{theorem}
	For $K=3$, and $\norm{x_k}_2=1, k=1,2,3$, the  expected sample complexity is lower bounded as
	$\frac{\expect{\tau}}{\log\frac{1}{2.4\delta}} \geq \Bigg(1+\frac{2\sin(\omega)}{\Delta_{min}} +\frac{\sin^2(\omega)}{\Delta_{min}^2} \Bigg). $
\end{theorem}
\section{Experiments}\label{Experiments}
In this section, we compare the performance of GLUCB with XY-static \cite{soare}, LUCB \cite{Kalyanakrishnan2012PACSS}, LinGapE \cite{XuAISTATS}  and X-ElimTilp with $p=0$ \cite{Tao_et-al-ICML-2018} through experiments in three synthetic settings and simulations based on real data. For LinGapE we implement the version of the algorithm which has been analyzed in their paper. For implementation of X-ElimTil0 with the setting as mentioned in \cite{Tao_et-al-ICML-2018}.
\subsection{Experiments based on synthetic data}
Throughout we assume noise $\epsilon_t\sim \mathcal{N}(0,1)$ independent. The results reported are averaged over 100 trials under each setting. We report the average number of samples to stop in each case, for each algorithm. The empirical probability of error in each case was found to be 0.
\begin{enumerate}
	\item Dataset 1: This is the setting introduced by \cite{soare} for linear bandits. We set up the linear bandit problem with $d+1$ arms, where features are the canonical bases $\{e_1,e_2,\ldots,e_d\}$ and an additional arm $x_{d+1}=(\cos(0.1),\sin(0.1),0,\ldots,0)$ with $\theta^*=(1,0,\ldots,0)$ so that the first arm is the best arm, with the $d+1$th arm being the most ambigous arm. We test by varying $d=2,\dots,10$. With $d=2$, this setup resembles the case we analyzed in \ref{upperbound}.
	\item Dataset 2: In this dataset, $K=100$ feature vectors  are sampled uniformly at random the surface of the unit sphere $\mathbb{S}^{d-1}$ centered at the origin. We pick the two closest arms, say  $u$ and $v$, and then set $\theta^*=u+\gamma(v-u)$ for $\gamma=0.01$. This makes $u$ as the best arm. We test the algorithm for $d=10,20,\ldots,100$.
	\item Dataset 3: This setup is important as this shows the efficiency of GLUCB in the case when there may be many arms which competing for the second best arm. For a given value of $K\geq3$, the armset contains feature vectors from $\mathbb{R}^2$ where $\mathcal{X}\equiv\{e_1, \cos(3\pi/4)e_1+\sin(3\pi/4)e_2 \}\cup\{\cos(\pi/4+\phi_k)e_1+\sin(\pi/4+\phi_k)e_2\}_{k=3}^{K}$ where $\phi_k\sim \mathcal{N}(0,0.09)$. $\theta^*$ was fixed to be $e_1$. We conduct the experiment by varying $K=\{10000,15000,\ldots,25000\}$.
\end{enumerate}
\subsection{Experiments based on real data}
We conduct an experiment on Yahoo! Webscope dataset R6A\footnote{\hyperbaseurl{https://webscope.sandbox.yahoo.com/}https://webscope.sandbox.yahoo.com/} which consists of features of 36-dimensions accompanied with binary outcomes. We  change the situation as is done in \cite{XuAISTATS} so that it can be adopted for best arm identification setting. We construct the 36-dimensional feature set $\mathcal{X}$ by the random sampling from the dataset, and the reward is generated as $y_t = 1$ with probability $(1+x_{a_t}^T\theta^*)/2$ and $-1$ with probability $(1-x_{a_t}^T\theta^*)/2$, %
where $ \theta^* $ is the regularized least squared estimator fitted for the original dataset. We choose the vectors such that $0<(1+{x_{a_t}}^T{\theta^*})/2<1$. For the detailed procedure, we refer the reader to the paper of \citet{XuAISTATS}.
\begin{table}[h]
	\parbox{.55\linewidth}{
  \caption{Synthetic dataset 1 $ \omega = 0.1 $, 100 trials, $\epsilon=0$}
  \label{Synthetic dataset 1 }
  \centering
  \begin{tabular}{lllll}
    \toprule
    $d$     & G-LUCB     & LinGapE   & X-ElimTil_{0} & LUCB \\
    \midrule
    2 & \textbf{12192}  & 13614 & 43680 & 2323124     \\
    3 & \textbf{15747}  & 16246 & 65520 & 2477473     \\
    4 & \textbf{18076}  & 19963 & 87360 & 2508303     \\
    5 & {22798}  & 21092 & 109200 & 2594715     \\
    6 & \textbf{24136}   & 24136 & 131040 & 2520462     \\
    7 & \textbf{25342}   & 27283 & 152880 & 3083729     \\
    8 & {29123}  & 28623 & 174720 & 3135822     \\
    9 & {32039}  & 31395 & 196560 & 3228290     \\
    10 & \textbf{33668}   & 34325  & 218400 & 3070843     \\
    \bottomrule
  \end{tabular}
}
\quad\quad
\parbox{0.43\linewidth}{
\begin{minipage}{0.9\linewidth}
			\centering
			\caption{Synthetic dataset 2,  100 trials, $\Delta_{min}>0.05$}
\label{Synthetic dataset 2}
			\includegraphics[scale=0.4]{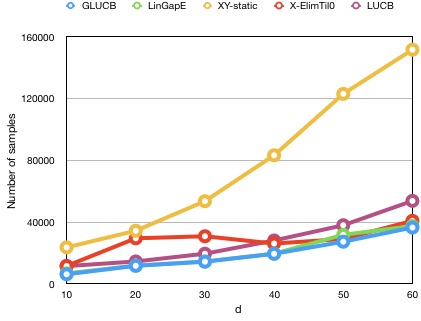}

		\end{minipage}
		}
\end{table}

	\begin{table}[htbp]
		\begin{minipage}{0.4\linewidth}
			\caption{Synthetic dataset 3, 100 trials}
			\label{Synthetic dataset 3}
			\centering
			\begin{tabular}{llll}
				\toprule
				$K$     & GLUCB    & LinGapE & X-ElimTil0\\
				\midrule
				15000 & \textbf{856}  & 860  &  5477 \\
				20000 & \textbf{895}  & 1112 &     6103\\
				25000 & {953}  & 867  & 7790  \\
				30000 & \textbf{861}  & 995  & 8054 \\
				\bottomrule
			\end{tabular}
		\end{minipage}\hfill
		\begin{minipage}{0.7\linewidth}
			\centering
			\includegraphics[scale=0.15]{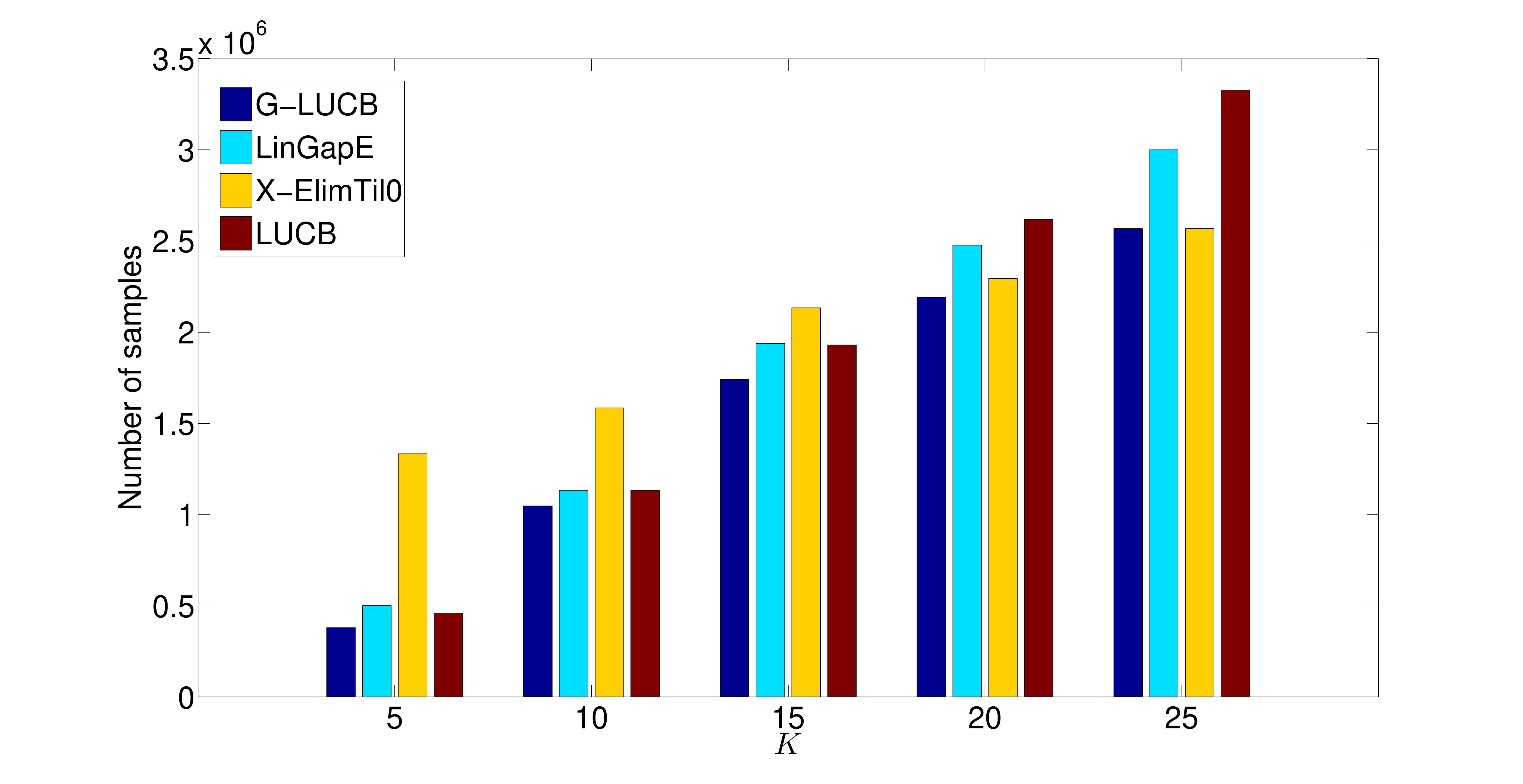}
\caption{Experiment on the Yahoo! dataset, 50 trials}
	\label{fig:AISTATS Yahoo!_bar}
		\end{minipage}
	\end{table}


%





\section{Conclusion and future work}

We have generalised the LUCB best arm identification algorithm to bandits with linear structure via a new MaxOverlap rule to reason under uncertainty. The resulting GLUCB algorithm is computationally very attractive as compared to many state-of-the-art algorithms for linear bandits. In particular, it does not require solving optimisation problems which are  inefficient when $K$ is large. Viewed from another perspective, the algorithm leverages the fact that a strategy which tries to greedily maximize the gap between the current best and second best arms is optimal for the BAI problem. We show that for the special case of two arms GLUCB is better than any causal algorithm for BAI problem. We also show orderwise optimality in the case of three arms. 
\par In light of the analysis presented and the performance of GLUCB in our experiments, we conjecture that our algorithm is optimal for any set of $K$ arms. Proving this forms part of our future work. We conduct several experiments based on synthetically designed environments and real world dataset and show the superior performance of GLUCB over other algorithms. Furthermore, we  believe that the factor of $d$ in the sample complexity results is due to  the general concentration bound for adaptive sequences and can be improved, which also remains as a future work. More generally, it is interesting to ask if the general max-overlap principle works for other, non-linear bandit reward structures as well. 
\bibliography{main}
\bibliographystyle{plainnat}

\newpage
\section{Appendix}

\subsection{Proof of Prop.~\ref{Prop:MaxOverlapArm}}\label{sec:appendixProofOfMaxOverlapArm}
As mentioned in Sec.~\ref{sec:TheGlucbAlgorithm}, at time $t+1$, our algorithm samples arm  
   \begin{align*}
   a_{t+1}&:= \argmin\limits_{a=1}^K MaxOverlap\Bigg(\mathcal{E}(\theta_t, V_t+x_ax_a^T); R_{h_t}^c\Bigg)
   \\
   &= \argmin\limits_{a=1}^K \max\limits_{\substack{\theta\in\mathcal{E} (\theta_t,V_t+ x_ax_a^T),\\ \theta\in R^c_{h_t}}}  \min\limits_{\tilde{\theta}\in \partial(R_{h_t}^c)} \norm{\theta-\tilde{\theta}}_2\\
   &= \argmin\limits_{a=1}^K \max\limits_{\substack{\theta\in\mathcal{E} (\theta_t,V_t+ x_ax_a^T),\\ \theta\in \bigcup\limits_{i \neq h_t} \{\theta':{\theta'}^Tx_i>{\theta'}^Tx_{h_t}\} }}  \min\limits_{\tilde{\theta}\in \bigcup\limits_{i \neq h_t} \{\theta':{\theta'}^Tx_i={\theta'}^Tx_{h_t}\}} \norm{\theta-\tilde{\theta}}_2\\
   &= \argmin\limits_{a=1}^K \max_{i\neq h_t}\max\limits_{\substack{\theta\in\mathcal{E} (\theta_t,V_t+ x_ax_a^T),\\ \theta^Tx_i>\theta^Tx_{h_t}}}  \min\limits_{\tilde{\theta}: \tilde{\theta}^Tx_i=\tilde{\theta}^Tx_{h_t}} \norm{\theta-\tilde{\theta}}_2
   \end{align*}
Where we have used the fact that $R_{h_t}^c=\sqcup_{i \neq h_t}R_i\equiv \bigcup\limits_{i\neq h_t}\Bigg\{ \tilde{\theta}\in \mathbb{R}^d: \tilde{\theta}^Tx_i>\tilde{\theta}^Tx_{h_t} \Bigg\}.$
To get a closed form solution for the above, we solve the optimization problem : $\max\limits_{i\neq h_t} \min\limits_{\tilde{\theta}:\tilde{\theta}^Tx_i > \tilde{\theta}^Tx_{h_t}}\norm{\theta-\tilde{\theta}}_2^2$
explicitly to obtain:
   \[a_{t+1}= \argmin\limits_{a=1}^{K} \max\limits_{i\neq h_t} \max\limits_{\theta \in \mathcal{E}(\theta_t, V_t+x_ax_a^T)} ({\theta^T(x_i-x_{h_t})})^+ \] 
The agent must stop choosing arms if this value is zero, i.e., when $\mathcal{C}_t$ no longer intersects any of the suboptimal cones.   
Due to the $max$ over $K-1$ arms the above strategy is inefficient to implement. A slight modification of the above which is easier to implement, is as follows. 
recalling the definition of $\text{HalfSpace}$, define the cone ($R(l_t)$) which has the maximum overlap with the \textit{current} cone.
\begin{align*}
l_t &:=\argmax\limits_{i\neq h_t, i\in [K]} MaxOverlap\Bigg(\mathcal{E}(\theta_t, V_t); \text{HalfSpace}(i,h_t) \Bigg) \\
&\equiv \argmax_{i\neq h_t}\max\limits_{\substack{\theta\in\mathcal{E} (\theta_t,V_t),\\ \theta^Tx_i>\theta^Tx_{h_t}}}  \min\limits_{\tilde{\theta}: \tilde{\theta}^Tx_i=\tilde{\theta}^Tx_{h_t}} \norm{\theta-\tilde{\theta}}_2\\
&\equiv \argmax\limits_{i\neq h_t} \max\limits_{\theta \in \mathcal{E}(\theta_t, V_t)}\theta^T(x_i-x_{h_t})\\
&=\argmax\limits_{i\neq h_t}\Bigg(\theta_t^T(x_i-x_{h_t}) + \beta_t\norm{x_i-x_{h_t}}_{V_t^{-1}} \Bigg),
\end{align*}
which is straightforward to implement on a computer.
Then play an arm according to the following rule. If $\max\limits_{\theta \in \mathcal{E}(\theta_t, V_t)}\theta^T(x_{l_t}-x_{h_t}) > 0,$ play:
\allowdisplaybreaks
\begin{align*}
a_{t+1} :=& \argmin\limits_{a \in [K]} MaxOverlap\Bigg(\mathcal{E}(\theta_t, V_t+x_ax_a^T); \text{HalfSpace}(l_t,h_t) \Bigg) \\
\equiv&\argmin\limits_{a=1}^{K} \max\limits_{\substack{\theta \in \mathcal{E}(\theta_t, V_t+x_ax_a^T)\\\theta^Tx_{l_t}>\theta^Tx_{h_t}}} \min\limits_{\substack{\tilde{\theta}\in \mathbb{R}^d:\\ \tilde{\theta}^Tx_{l_t}=\tilde{\theta}^Tx_{h_t}}} \norm{\theta-\tilde{\theta}}_2\\
=& \argmin\limits_{a=1}^{K} \max\limits_{\theta \in \mathcal{E}(\theta_t, V_t+x_ax_a^T)} \theta^T(x_{l_t}-x_{h_t})\\
=& \argmin\limits_{a=1}^{K}  \Bigg[ \theta_t^T(x_{l_t}-x_{h_t}) +\beta_t\norm{x_{l_t}-x_{h_t}}_{(V_t+x_ax_a^T)^{-1}}\Bigg] \\
=& \argmin\limits_{a=1}^{K}  \norm{x_{l_t}-x_{h_t}}_{(V_t+x_ax_a^T)^{-1}} \\
\equiv& \argmax\limits_{a=1}^K \frac{\abs{x_a^TV_t^{-1}(x_{l_t}-x_{h_t})}}{\sqrt{1+ \norm{x_a}^2_{V_t^{-1}}}}.
\end{align*}
 The last step uses the Matrix Inversion Lemma \cite{Horn:2012:MA:2422911}.
\subsection{Lower Bound on Sample Complexity}\label{proof of lower bound}

We begin by restating the result of [\cite{KaufmanOpt}] in the special case of linear bandits. Let $\{x_1,x_2,\ldots,x_K\}$ be a given set of arms  in $\mathbb{R}^d$. Let $\theta$ be any vector in $\mathbb{R}^d$.  For any arbitrary vector $v\in \mathbb{R}^d$, we define $a^*(v):=\argmax\limits_{x_a\in \mathcal{X}}v^Tx_a$. Define the set $Alt(\theta) := \{\xi \in \mathbb{R}^d: a^*(\xi)\neq a^*(\theta) \}$. 
  \begin{lemma}[General change-of-measure based lower bound of \citet{KaufmanOpt}]\label{Kaufmann}
  Let $\delta \in (0,1).$ For any $\delta-PAC$ strategy and any linear bandit with the unknown parameter vector $\theta \in \mathbb{R}^d$. Let the noise be normal with a variance parameter of $1/2$. Then the expected sample complexity of any strategy
  \beqn
  \mathbb{E}_{\theta}[\tau_{\delta}] \geq T^*(\theta) kl(\delta,1-\delta)
  \eeqn
  where, $kl(x,y) = x\log\frac{x}{y}+ (1-x)\log\frac{1-x}{1-y}$ for $x,y\in [0,1]$ and
  \[T^*(\theta)^{-1} := \sup\limits_{\underline{w}\in \mathcal{P}_K}\inf\limits_{\xi\in Alt(\theta)}\sum\limits_{a=1}^{K}w_aKL(\mathcal{N}(\theta^Tx_a,1/2),\mathcal{N}(\xi^Tx_a,1/2)),\]
 where $\mathcal{P}_K$ is defined as the set of all probability distributions on $\mathcal{X}$ and $KL(P,Q)$ is the KL-divergence between any two probability distributions $P$ and $Q$, and $\mathcal{N}\left(\mu,\sigma^2\right)$ is normal distribution with mean $\mu$ and variance $\sigma^2$.
  \end{lemma}
  
\begin{theorem*}[Lower Bound]
Let $\delta \in (0,0.15).$ For any $\delta-PAC$ strategy and any linear bandit problem with the unknown parameter vector $\theta \in \mathbb{R}^d$, 
\beqn
\expect{\tau_{\delta}} \geq \min\limits_{\underline{w}\in \mathcal{P}_K}\max\limits_{a\neq a^*}\frac{\norm{x_a-x_1}^2_{W^{-1}}}{\Delta_{a}^2}\log\frac{1}{2.4\delta}
\eeqn
where the expectation is under $\theta$ and $W:=\sum\limits_{a=1}^{K}w_ax_ax_a^T$, where $\{x_1,x_2,\ldots,x_K\}$ are the arms in $\mathbb{R}^d.$
\end{theorem*}
\begin{proof}{}
Recall the definition from Lemma \ref{Kaufmann}, 
\beqn
T^*(\theta)^{-1} = \sup\limits_{\underline{w}\in \mathcal{P}_K}\inf\limits_{\xi\in Alt(\theta)}\sum\limits_{a=1}^{K}w_ad(\mathcal{N}(\theta^Tx_a,1/2),\mathcal{N}(\xi^Tx_a,1/2)),
\eeqn
where $d(.,.)$ is the KL divergence between any two distributions. Hence, we have,
\begin{align*}
T^*(\theta)^{-1} =& \sup\limits_{\underline{w}\in \mathcal{P}_K}\inf\limits_{\xi\in Alt(\theta)}\sum\limits_{a=1}^{K}w_a(\theta^Tx_a - \xi^Tx_a)^2\\
=&  \sup\limits_{\underline{w}\in \mathcal{P}_K}\min\limits_{a'\neq 1}\inf\limits_{\xi: \xi^Tx_a' \geq \xi^Tx_1+\epsilon}\sum\limits_{a=1}^{K}w_a(\theta^Tx_a - \xi^Tx_a)^2
\end{align*}
for some $\epsilon >0.$  We will first consider the inner part of the expression above (a convex program), which can be re-written as 
\beq
\inf\limits_{\xi: \xi^Tx_a' \geq \xi^Tx_1+\epsilon}(\theta-\xi)^TW(\theta-\xi)
\eeq
where $W$ is defined as in the theorem. Writing the Lagrangian, we get
\beqn
\mathcal{L}(\xi, \lambda) = (\theta-\xi)^TW(\theta-\xi) -\lambda(\xi^Tx_a' - \xi^Tx_1-\epsilon).
\eeqn
Setting $\Delta_{\xi}\mathcal{L}(\xi, \lambda)=0$, we get
\begin{align*}
\Delta_{\xi}\mathcal{L}(\xi, \lambda):& -2W\theta + 2W\xi -\lambda(x_{a'}-x_1)= 0\\
\Rightarrow \xi &= \theta +\frac{\lambda}{2}W^{-1}(x_{a'}-x_1).
\end{align*}
Substituting this value of $\xi$ into the Lagrangian, we obtain:
\begin{flalign*}
\begin{split}
 \mathcal{L}(\lambda)&= \theta^TW\theta -2\theta^TW(\theta +\frac{\lambda}{2}W^{-1}(x_{a'}-x_1)) \\& +(\theta +\frac{\lambda}{2}W^{-1}(x_{a'}-x_1))^TW(\theta +\frac{\lambda}{2}W^{-1}(x_{a'}-x_1)) \\& -\lambda((\theta +\frac{\lambda}{2}W^{-1}(x_{a'}-x_1))^T(x_{a'}-x_1)-\epsilon)\\
\end{split}\\
&=-\frac{\lambda^2}{4} (x_{a'}-x_1)^TW^{-1}(x_{a'}-x_1) -\lambda(\theta^T(x_{a'}-x_1)) + \lambda \epsilon.
\end{flalign*}
Setting $\Delta_{\lambda}\mathcal{L}(\lambda) =0,$
\beq
\frac{\lambda}{2} = \frac{\epsilon +\Delta_{a'}}{\norm{(x_{a'}-x_1)}_{W^{-1}}^2}.
\eeq
Hence, we get :
\beq
\xi^* = \theta + \frac{\epsilon +\Delta_{a'}}{\norm{(x_{a'}-x_1)}_{W^{-1}}^2}W^{-1}(x_{a'}-x_1).
\eeq
Using this value of $\xi$ in the objective we get,
\begin{align*}
T^*(\theta)^{-1} &= \sup\limits_{\underline{w}\in \mathcal{P}_K}\min\limits_{a'\neq 1} \frac{(\epsilon +\Delta_{a'})^2}{\norm{(x_{a'}-x_1)}_{W^{-1}}^2}.
\end{align*}
Putting this into Lemma \ref{Kaufmann},  and taking $\epsilon \to 0$, completes the proof.
\end{proof}

\subsection{Proof of lemma \ref{lemma:convexity of y'V-1y}}\label{proof of convexity}
\begin{proof}
 We will show that $\forall t\in [0,1]$, and for any $\lambda_1,\lambda_2\in \mathcal{P}_K$,
 \[y^T(\sum\limits_{i=1}^{K}(t\lambda_1(i)+(1-t)\lambda_2(i))x_ix_i ^T)^{-1}y \leq  ty^T(\sum\limits_{i=1}^{K}(\lambda_1(i))x_ix_i^T )^{-1}y +(1-t)y^T(\sum\limits_{i=1}^{K}(\lambda_2(i))x_ix_i )^{-1}y.\] 
Let us define $Z(t):=\sum\limits_{i=1}^{K}(t\lambda_1(i)+(1-t)\lambda_2(i))x_ix_i ^T=t\sum\limits_{i=1}^{K}\lambda_1(i)x_ix_i^T+(1-t)\sum\limits_{i=1}^{K}\lambda_2(i)x_ix_i ^T.$
 Let $(Z(t))'$ denote the derivative of (Z(t)) with respect to $t$. Clearly, for $t\in[0,1]$, $Z(t)$ is positive definite. Also,
 \begin{equation}\label{eq:conv}
 Z(t)Z(t)^{-1}=I= Z(t)'Z(t)^{-1}+ Z(t)(Z(t)^{-1})'=0\Rightarrow (Z(t)^{-1})'=-Z(t)^{-1}Z(t)'Z(t)^{-1}. 
 \end{equation}
 Differentiating one more time and noticing that $Z(t)''=0$, we get:
 \[(Z(t)^{-1})= 2Z(t)^{-1}Z(t)'Z(t)^{-1}Z(t)'Z(t)^{-1}. \]
 For any $u\in \mathbb{R}^d$, we have $v(t)=Z(t)'Z(t)^{-1}u$. Let $\psi(t)= u^TZ(t)^{-1}u.$  From \ref{eq:conv} we get $\psi(t)''= u^T{Z(t)^{-1}}''u=2v(t)^TZ(t)^{-1}v(t)\geq 0$ since $Z(t)^{-1}$ is PSD. From this we conclude that $\psi(t)$ is convex over $t\in [0,1].$
 As a result for any $t'\in [0,1]$ we have \[(1-t)\psi(0)+t\psi(1)-\psi(t)\geq 0\]
 \[\Leftrightarrow u^T[(1-t')(\sum\limits_{i=1}^{K}(\lambda_1(i))x_ix_i^T )^{-1} +t'(\sum\limits_{i=1}^{K}(\lambda_2(i))x_ix_i )^{-1}-(\sum\limits_{i=1}^{K}((1-t')\lambda_1(i)+t'\lambda_2(i))x_ix_i ^T)^{-1}]u\geq 0.\]
 
 Since $u$ is arbitrary it means the matrix inside the bracket is positive definite, which implies convexity of the desired function.

\end{proof}

\subsection{Proof of theorem \ref{2armupperbound}}\label{proof of 2 arms}
\begin{proof}{Proof of lemma \ref{lemma:2arm1stlemma}}
Let $n_1(t)=n_2(t)=n.$ Let $A=\lambda I +nx_1x_1^T+nx_2x_2^T.$  By Sherman-Morrison-Woodbury identity for matrix inversion and matrix associativity we have, 
\begin{align}
A^{-1} =&\Big(\frac{I}{\lambda} -\frac{nx_1x_1^T}{\lambda(\lambda+n)}\Big) -\frac{\Big(\frac{I}{\lambda} -\frac{nx_1x_1^T}{\lambda(\lambda+n)}\Big)nx_2x_2^T\Big(\frac{I}{\lambda} -\frac{nx_1x_1^T}{\lambda(\lambda+n)}\Big)}{1+nx_2^T\Big(\frac{I}{\lambda} -\frac{nx_1x_1^T}{\lambda(\lambda+n)}\Big)x_2}\\
=&\Big(\frac{I}{\lambda} -\frac{nx_2x_2^T}{\lambda(\lambda+n)}\Big) -\frac{\Big(\frac{I}{\lambda} -\frac{nx_2x_2^T}{\lambda(\lambda+n)}\Big)nx_1x_1^T\Big(\frac{I}{\lambda} -\frac{nx_2x_2^T}{\lambda(\lambda+n)}\Big)}{1+nx_1^T\Big(\frac{I}{\lambda} -\frac{nx_2x_2^T}{\lambda(\lambda+n)}\Big)x_1}
\end{align}
At time $t+1$ algorithm chooses: $a_{t+1}:= \argmax\limits_{a=1,2}\Big\{\frac{(x_a^TA^{-1}(x_1-x_2))^2}{1+x_a^TA^{-1}x_a} \Big\}.$ Let's calculate the terms explicitly.
By using the first formulation of $A^{-1}$ and the fact that $\norm{x}=1$, we get:
\begin{align*}
x_1^TA^{-1}x_1=&\Big(\frac{1}{\lambda} -\frac{n}{\lambda(\lambda+n)}\Big) -\frac{\Big(\frac{x_1^T}{\lambda} -\frac{nx_1^T}{\lambda(\lambda+n)}\Big)nx_2x_2^T\Big(\frac{x_1}{\lambda} -\frac{nx_1}{\lambda(\lambda+n)}\Big)}{1+\frac{n}{\lambda} -\frac{n^2(x_1^Tx_2)^2}{\lambda(\lambda+n)}}\\
=&\Big(\frac{1}{\lambda} -\frac{n}{\lambda(\lambda+n)}\Big) -\frac{n(1-\rho)^2\Big(\frac{1}{\lambda} -\frac{n}{\lambda(\lambda + n)} \Big)}{1+\frac{n}{\lambda} -\frac{n^2(1-\rho)^2}{\lambda(\lambda+n)}}
\end{align*}
Next we show by using the second formulation of $A^{-1}$, that
\begin{align*}
x_2^TA^{-1}x_2=&\Big(\frac{1}{\lambda} -\frac{n}{\lambda(\lambda+n)}\Big) -\frac{\Big(\frac{x_2^T}{\lambda} -\frac{nx_2^T}{\lambda(\lambda+n)}\Big)nx_1x_1^T\Big(\frac{x_2}{\lambda} -\frac{nx_2}{\lambda(\lambda+n)}\Big)}{1+\frac{n}{\lambda} -\frac{n^2(x_1^Tx_2)^2}{\lambda(\lambda+n)}}\\
=&\Big(\frac{1}{\lambda} -\frac{n}{\lambda(\lambda+n)}\Big) -\frac{n(1-\rho)^2\Big(\frac{1}{\lambda} -\frac{n}{\lambda(\lambda + n)} \Big)}{1+\frac{n}{\lambda} -\frac{n^2(1-\rho)^2}{\lambda(\lambda+n)}}
\end{align*}
We just showed that if $A=\lambda I +nx_1x_1^T+nx_2x_2^T,$ then we can $x_1^TA^{-1}x_1=x_2^TA^{-1}x_2.$
Next we show that there is indeed a tie. However, that is clear from the above part as:
\begin{align*}
\frac{(x_1^TA^{-1}(x_1-x_2))^2}{1+x_1A^{-1}x_1} =&
\frac{(x_1^TA^{-1}x_1-x_1^TA^{-1}x_2))^2}{1+x_1A^{-1}x_1}\\ 
=& \frac{(x_2^TA^{-1}x_2-x_2^TA^{-1}x_1))^2}{1+x_2A^{-1}x_2}\\
=& \frac{(x_2^TA^{-1}(x_1-x_2))^2}{1+x_2A^{-1}x_2}.
\end{align*}
\end{proof}
\begin{proof}{Proof of lemma \ref{lemma:2arm2ndlemma}}
We start by solving the LHS. 

\[\frac{(x_1^TA^{-1}(x_1-x_2))^2}{1+x_1A^{-1}x_1} = \frac{(x_1^TA^{-1}x_1-x_1^TA^{-1}x_2)^2}{1+x_1A^{-1}x_1}\\
= \frac{(x_1^TA^{-1}x_1)^2}{1+x_1A^{-1}x_1}\Big( 1- \frac{x_1^TA^{-1}x_2}{x_1^TA^{-1}x_1}\Big).\]

\textbf{Claim 1.} $x_1^TA^{-1}x_2 \leq x_k^TA^{-1}x_k$, where $k=1,2.$\\
\textit{Proof of claim.} 
Denote $A = B + x_1x_1^T$, where $B = \lambda I +nx_1x_1^T+nx_2x_2^T.$
\[
x_1^TA^{-1}x_2 = x_1^T\Big(B^{-1}-\frac{B^{-1}x_1x_1^TB^{-1}}{1+x_1^TB^{-1}x_1} \Big)x_2\\
= \frac{x_1^TB^{-1}x_2}{1+x_1^TB^{-1}x_1}
=^{*} \frac{x_1^TB^{-1}x_2}{1+x_2^TB^{-1}x_2}.\]
The last equality follows from lemma \ref{lemma:2arm1stlemma}(i).
Hence we will be done if we show $x_1^TB^{-1}x_2 \leq x_1^TB^{-1}x_1.$ But this follows from Cauchy-Schwarz as follows:

\[x_1^TB^{-1}x_2 = \langle B^{-1/2}x_1, B^{-1/2}x_2\rangle\\
\leq  \sqrt{x_1^TB^{-1}x_1}\sqrt{x_2^TB^{-1}x_2}\\
=^{*} x_k^TB^{-1}x_k, (k=1,2).\]
Again The last equality follows from lemma \ref{lemma:2arm1stlemma}(i).This completes the proof of claim 1.
Also, we can easily check that $x_1^TA^{-1}x_1 \leq x_2^TA^{-1}x_2.$
Next we observe that $\frac{y^2}{1+y}$ is a monotone increasing function for $y\geq 0. $ Hence combining this fact and  claim 1, we can have the following upper bound on $\frac{(x_1^TA^{-1}(x_1-x_2))^2}{1+x_1A^{-1}x_1}$:
\begin{align*}
\frac{(x_1^TA^{-1}(x_1-x_2))^2}{1+x_1A^{-1}x_1} \leq \frac{(x_2^TA^{-1}x_2)^2}{1+x_2A^{-1}x_2}\Big( 1- \frac{x_1^TA^{-1}x_2}{x_1^TA^{-1}x_1}\Big)^2
=\frac{\Big(x_2^TA^{-1}x_2 - \frac{x_2^TA^{-1}x_2}{x_1^TA^{-1}x_1}x_1^TA^{-1}x_2 \Big)^2}{1+x_2A^{-1}x_2}
\end{align*}
The numerator is of the form $(a-\alpha b)^2$ where $\alpha\geq 1.$ We will show here that $a\geq \alpha b$, i.e.,  $x_2^TA^{-1}x_2 \geq \frac{x_2^TA^{-1}x_2}{x_1^TA^{-1}x_1}x_1^TA^{-1}x_2$. But this is  equivalent to showing $ x_1^TA^{-1}x_1\geq x_1^TA^{-1}x_2$ as $x_2^TA^{-1}x_2>0$. But this again follows from claim  1. Hence we have $0\leq (a- \alpha b)\leq (a-b)$ as $\alpha\geq1.$ Hence, we can write:
\[\frac{(x_1^TA^{-1}(x_1-x_2))^2}{1+x_1A^{-1}x_1}\leq  \frac{(x_2^TA^{-1}x_2-x_1^TA^{-1}x_2)^2}{1+x_2A^{-1}x_2} = \frac{(x_2^TA^{-1}(x_1-x_2))^2}{1+x_2A^{-1}x_2}.\]
This completes the proof.
\end{proof}

\subsection{Proof of Theorem \ref{optimality of glucb for 2 arm}}\label{proof of 2 arms theorem}
Let us consider two algorithms: $\mathcal{A}_1$ and $\mathcal{A}_2$, where $\mathcal{A}_1$ is G-LUCB and $\mathcal{A}_2$ is any other algorithm. Then,
\begin{lemma}
$\forall t>0$, $\Phi^{\mathcal{A}_1}(t) \leq \Phi^{\mathcal{A}_2}(t)$, $\forall l\geq 0.$
\end{lemma}
\begin{proof}
Let the current round be ${t}$. Let $N_i^{\mathcal{A}_j}:=$ number of times arm $i$ has been been played by algorithm $\mathcal{A}_j$. Let ${t'}:= \min\{N_1^{\mathcal{A}_2}, N_2^{\mathcal{A}_2} \}$. Let us also, rearrange the plays of algorithm $\mathcal{A}_2$  such that for the first $t$ rounds, $\mathcal{A}_2$ plays arms 1 and 2 equal number of times. Let $A= I+ \frac{t'}{2}x_1x_1^T + \frac{t'}{2}x_2x_2^T$. Clearly, we have ${\norm{x_1-x_2}^2_{A^{-1}}}^{\mathcal{A}_1} = {\norm{x_1-x_2}^2_{A^{-1}}}^{\mathcal{A}_2}.$ Let $l=t-t'.$ Without loss of generality, let arm 2 be played more number of times in algorithm $\mathcal{A}_2$. To summarize, we have the following:
\begin{align*}
N_1^{\mathcal{A}_1}&=\frac{t}{2}=\frac{t'}{2}+\frac{l}{2},\\
N_2^{\mathcal{A}_1}&=\frac{t}{2}=\frac{t'}{2}+\frac{l}{2},\\
N_1^{\mathcal{A}_2}&=\frac{t'}{2},\\
N_2^{\mathcal{A}_2}&=\frac{t'}{2}+l,\\
V_t^{\mathcal{A}_1} &= A+\frac{l}{2}x_1x_1^T+ \frac{l}{2}x_2x_2^T,\\
V_t^{\mathcal{A}_1} &= A+lx_2x_2^T.
\end{align*}

Let us now calculate the potential functions for both the algorithms. 
\begin{align*}
\Phi^{\mathcal{A}_2}:=& \norm{x_1-x_2}^2_{V_t^{{\mathcal{A}_2}^{-1}}}=\norm{x_1-x_2}^2_{(A+lx_2x_2)^{-1}} \\
=&(x_1-x_2)^T\Bigg(A^{-1} - \frac{lA^{-1}x_2x_2^TA^{-1}}{1+lx_2^TA^{-1}x_2} \Bigg)(x_1-x_2)\\
=& \norm{x_1-x_2}^2_{A^{-1}} - \frac{l(x_2^TA^{-1}(x_1-x_2))^2}{1+lx_2^TA^{-1}x_2}.
\end{align*}
Next for algorithm 1,
\begin{align*}
\allowdisplaybreaks
\Phi^{\mathcal{A}_1}:=& \norm{x_1-x_2}^2_{V_t^{{\mathcal{A}_1}^{-1}}}=\norm{x_1-x_2}^2_{( A+\frac{l}{2}x_1x_1^T+ \frac{l}{2}x_2x_2^T)^{-1}} \\
=& \norm{x_1-x_2}^2_{(A+\frac{l}{2}x_1x_1^T)^{-1}} - \frac{l}{2}\frac{(x_2^T(A+\frac{l}{2}x_1x_1^T)^{-1}(x_1-x_2))^2}{1+\frac{l}{2}x_2^T(A+\frac{l}{2}x_1x_1^T)^{-1}x_2}\\
 =& \norm{x_1-x_2}^2_{A^{-1}} -\frac{l}{2}\frac{(x_1^TA^{-1}(x_1-x_2))^2}{1+\frac{l}{2}x_1^TA^{-1}x_1} - \frac{l}{2}\frac{\Bigg(x_2^T\Big(A^{-1}-\frac{\frac{l}{2}A^{-1}x_1x_1^TA^{-1}}{1+\frac{l}{2}x_1^TA^{-1}x_1} \Big)(x_1-x_2) \Bigg)^2}{1+\frac{l}{2}x_2^T\Big(A^{-1}-\frac{\frac{l}{2}A^{-1}x_1x_1^TA^{-1}}{1+\frac{l}{2}x_1^TA^{-1}x_1} \Big) x_2}\\
 =& \norm{x_1-x_2}^2_{A^{-1}} -\frac{l}{2}\frac{(x_1^TA^{-1}(x_1-x_2))^2}{1+\frac{l}{2}x_1^TA^{-1}x_1} - \frac{l}{2}\frac{\Bigg(x_2^TA^{-1}(x_1-x_2)-\frac{l}{2}\frac{x_2^TA^{-1}x_1x_1^TA^{-1}(x_1-x_2)}{1+\frac{l}{2}x_1^TA^{-1}x_1} \Bigg)^2}{1+\frac{l}{2}x_2^TA^{-1}x_2-\frac{l^2}{4}\frac{(x_1^TA^{-1}x_2)^2}{1+\frac{l}{2}x_1^TA^{-1}x_1}}\\
 =^{lemma \ref{lemma:2arm1stlemma} }& \norm{x_1-x_2}^2_{A^{-1}} -\frac{l}{2}\frac{(x_1^TA^{-1}(x_1-x_2))^2}{1+\frac{l}{2}x_1^TA^{-1}x_1} - \frac{l}{2}\frac{(x_1^TA^{-1}(x_1-x_2))^2\Bigg(1+\frac{l}{2}\frac{x_1^TA^{-1}x_2}{1+\frac{l}{2}x_1^TA^{-1}x_1} \Bigg)^2}{1+\frac{l}{2}x_1^TA^{-1}x_1-\frac{l^2}{4}\frac{(x_1^TA^{-1}x_2)^2}{1+\frac{l}{2}x_1^TA^{-1}x_1}}\\
 \leq^{(*)} & \norm{x_1-x_2}^2_{A^{-1}} -\frac{l}{2}\frac{(x_1^TA^{-1}(x_1-x_2))^2}{1+\frac{l}{2}x_1^TA^{-1}x_1} 
 -\frac{l}{2}\frac{(x_1^TA^{-1}(x_1-x_2))^2}{1+\frac{l}{2}x_1^TA^{-1}x_1}\\
 \leq & \norm{x_1-x_2}^2_{A^{-1}} -l\frac{(x_1^TA^{-1}(x_1-x_2))^2}{1+lx_1^TA^{-1}x_1} = \Phi^{\mathcal{A}_2}.
\end{align*}
\end{proof}
\subsection{Three arm analysis}
\subsubsection{Proof of lemma \ref{lemma:key3armcase}}\label{proof of key 3 arm lemma}
\begin{proof}
We will show the result by induction on $t$. We begin with the base case when $t=1$. At $t=0$, $V_0=\lambda I$ and $\{h_*^t,l_*^t \}$ can be any one of the $3\choose 2$ possible combinations. Recall that, at any time $t$ the algorithm plays 
\[a_{t}=\argmax\limits_{a\in [K]}\frac{\abs{x_a^TV_{t-1}^{-1}(x_{h_*^t}-x_{l_*^t})}}{\sqrt{1+\norm{x_a}_{V_{t-1}^{-1}}}}. \]
\textbf{Base case: $t=1$}\\
\par \textit{Sub-case 1: $\{h_*^t,l_*^t\}\in \{1,2\}$.} We can easily calculate the arguments in the argmax in this case:
Play\[\argmax\Bigg\{ \frac{1}{\sqrt{\lambda(\lambda+1)}}, \frac{1}{\sqrt{\lambda(\lambda+1)}}, \frac{\abs{\cos(\omega)-\sin(\omega)}}{\sqrt{\lambda(\lambda+1)}} \Bigg\} .\] 
By our assumption on the range of $\omega$, argmax is satisfied for arm 1 and 2. Hence for this sub-case at $t=1$, arm 3 is \textit{not} played.
\par \textit{Sub-case 2: $\{h_*^t,l_*^t\}\in \{2,3\}$.} We can again easily calculate the arguments in the argmax in this case:
Play\[\argmax\Bigg\{ \frac{\abs{\cos(\omega)}}{\sqrt{\lambda(\lambda+1)}}, \frac{1-\sin(\omega)}{\sqrt{\lambda(\lambda+1)}}, \frac{1-\sin(\omega)}{\sqrt{\lambda(\lambda+1)}} \Bigg\} .\] 
Again, by the assumption on the range of $\omega$, argmax is satisfied for arm 1. Hence again for this sub-case at $t=1$, arm 3 is \textit{not} played.

\par \textit{Sub-case 3: $\{h_*^t,l_*^t\}\in \{1,3\}$.} Finally, we  calculate the arguments in the argmax for this case:
Play\[\argmax\Bigg\{ \frac{\abs{1-\cos(\omega)}}{\sqrt{\lambda(\lambda+1)}}, \frac{\sin(\omega)}{\sqrt{\lambda(\lambda+1)}}, \frac{\abs{1-\cos(\omega)}}{\sqrt{\lambda(\lambda+1)}} \Bigg\} .\] 
Again, by the assumption on the range of $\omega$, argmax is satisfied for arm 2. Hence again for this sub-case at $t=1$, arm 3 is \textit{not} played.
\par We have shown that arm 3 will \textit{not} be played at round 1.

\textbf{General: $t>1$}\\
Let us now study the behavior of G-LUCB, for a general $t>1$, for this specific case. Let at round $t-1$, $V_{t-1}=\lambda I +n_1e_1e_1^T+ n_2e_2e_2^T$, where $n_1,n_2\geq 0$, integers.Again, we divide our analysis into three sub-cases depending on the realizations of the set $\{h_*^t,l_*^t \}$.  

\par \textit{Sub-case 1: $\{h_*^t,l_*^t\}\in \{1,2\}$.} After some easy calculations we end up at determining:
\[\argmax\Bigg\{ \frac{1}{\sqrt{(\lambda+n_1)(\lambda+n_1+1)}}, \frac{1}{\sqrt{(\lambda+n_2)(\lambda+n_2+1)}} , \frac{\abs{\frac{\cos(\omega)}{\lambda+n_1}-\frac{\sin(\omega)}{\lambda+n_2}}}{\sqrt{1+\frac{\cos^2(\omega)}{\lambda+n_1}+\frac{\sin^2(\omega)}{\lambda+n_2}}}         \Bigg \} \]
We have 2 cases depending on the sign of the last term. \\
\underline{If $ \frac{\cos(\omega)}{\lambda+n_1}\geq\frac{\sin(\omega)}{\lambda+n_2} $}: The term term becomes:
\begin{align*}
\frac{\abs{\frac{\cos(\omega)}{\lambda+n_1}-\frac{\sin(\omega)}{\lambda+n_2}}}{\sqrt{1+\frac{\cos^2(\omega)}{\lambda+n_1}+\frac{\sin^2(\omega)}{\lambda+n_2}}}=& \frac{{\frac{\cos(\omega)}{\lambda+n_1}-\frac{\sin(\omega)}{\lambda+n_2}}}{\sqrt{1+\frac{\cos^2(\omega)}{\lambda+n_1}+\frac{\sin^2(\omega)}{\lambda+n_2}}}\\
<& \frac{{\frac{\cos(\omega)}{\lambda+n_1}}}{\sqrt{1+\frac{\cos^2(\omega)}{\lambda+n_1}+\frac{\sin^2(\omega)}{\lambda+n_2}}}\\
<& \frac{{\frac{\cos(\omega)}{\lambda+n_1}}}{\sqrt{1+\frac{\cos^2(\omega)}{\lambda+n_1}}}\\=&\frac{\cos(\omega)}{\sqrt{(\lambda+n_1)(\lambda+n_1+\cos^2(\omega))}}\\
\leq&  \frac{1}{\sqrt{(\lambda+n_1)(\lambda+n_1+1)}}
\end{align*}
where the last inequality follows by noting that $\frac{x}{\sqrt{a+x^2}}$ is increasing in $x\geq0$ for $a>0.$ 
\par \underline{Else if $ \frac{\cos(\omega)}{\lambda+n_1}<\frac{\sin(\omega)}{\lambda+n_2} $}: The term term becomes:
\begin{align*}
\frac{\abs{\frac{\cos(\omega)}{\lambda+n_1}-\frac{\sin(\omega)}{\lambda+n_2}}}{\sqrt{1+\frac{\cos^2(\omega)}{\lambda+n_1}+\frac{\sin^2(\omega)}{\lambda+n_2}}}=& \frac{\frac{\sin(\omega)}{\lambda+n_2}-{\frac{\cos(\omega)}{\lambda+n_1}}}{\sqrt{1+\frac{\cos^2(\omega)}{\lambda+n_1}+\frac{\sin^2(\omega)}{\lambda+n_2}}}\\
<& \frac{\frac{\sin(\omega)}{\lambda+n_2}}{\sqrt{1+\frac{\cos^2(\omega)}{\lambda+n_1}+\frac{\sin^2(\omega)}{\lambda+n_2}}}\\
<& \frac{\frac{\sin(\omega)}{\lambda+n_2}}{\sqrt{1+\frac{\sin^2(\omega)}{\lambda+n_2}}}\\
=&\frac{\sin(\omega)}{\sqrt{(\lambda+n_2)(\lambda+n_2+\sin^2(\omega))}}\\
\leq&  \frac{1}{\sqrt{(\lambda+n_2)(\lambda+n_2+1)}}
\end{align*}
where, again, the last inequality follows by noting that $\frac{x}{\sqrt{a+x^2}}$ is increasing in $x\geq0$ for $a>0.$
\par Both of the above cases show that there exists some arm either 1 or 2, which always overpowers the third arm.

\par \textit{Sub-case 2: $\{h_*^t,l_*^t\}\in \{2,3\}$.} We will play:
\[\argmax\Bigg\{ \frac{\cos(\omega)}{\sqrt{(\lambda+n_1)(\lambda+n_1+1)}}, \frac{1-\sin(\omega)}{\sqrt{(\lambda+n_2)(\lambda+n_2+1)}} , \frac{\abs{\frac{\cos^2(\omega)}{\lambda+n_1}-\frac{\sin(\omega)(1-\sin(\omega))}{\lambda+n_2}}}{\sqrt{1+\frac{\cos^2(\omega)}{\lambda+n_1}+\frac{\sin^2(\omega)}{\lambda+n_2}}}         \Bigg \} \]
We have 2 cases depending on the sign of the last term. \\
\underline{If $ {\frac{\cos^2(\omega)}{\lambda+n_1}\geq \frac{\sin(\omega)(1-\sin(\omega))}{\lambda+n_2}} $}: The term  becomes:
\begin{align*}
\frac{\abs{\frac{\cos^2(\omega)}{\lambda+n_1}-\frac{\sin(\omega)(1-\sin(\omega))}{\lambda+n_2}}}{\sqrt{1+\frac{\cos^2(\omega)}{\lambda+n_1}+\frac{\sin^2(\omega)}{\lambda+n_2}}}=& \frac{\frac{\cos^2(\omega)}{\lambda+n_1}-\frac{\sin(\omega)(1-\sin(\omega))}{\lambda+n_2}}{\sqrt{1+\frac{\cos^2(\omega)}{\lambda+n_1}+\frac{\sin^2(\omega)}{\lambda+n_2}}}\\
\leq& \frac{{\frac{\cos^2(\omega)}{\lambda+n_1}}}{\sqrt{1+\frac{\cos^2(\omega)}{\lambda+n_1}+\frac{\sin^2(\omega)}{\lambda+n_2}}}\\
<& \frac{{\frac{\cos^2(\omega)}{\lambda+n_1}}}{\sqrt{1+\frac{\cos^2(\omega)}{\lambda+n_1}}}\\
=&\frac{\cos^2(\omega)}{\sqrt{(\lambda+n_1)(\lambda+n_1+\cos^2(\omega))}}\\
=&\cos(\omega)\Bigg(\frac{\cos(\omega)}{\sqrt{(\lambda+n_1)(\lambda+n_1+\cos^2(\omega))}}\Bigg)\\
\leq&  \cos(\omega)\frac{1}{\sqrt{(\lambda+n_1)(\lambda+n_1+1)}}
\end{align*}
where again we use the same inequality that $\frac{x}{\sqrt{a+x^2}}$ is increasing in $x\geq0$ for $a>0$, $\cos(\omega)\geq 0$ and $\cos^2(\omega)\leq 1.$ 
\par \underline{Else if $ {\frac{\cos^2(\omega)}{\lambda+n_1} < \frac{\sin(\omega)(1-\sin(\omega))}{\lambda+n_2}} $}: The third term  becomes:
\begin{align*}
\frac{\abs{{\frac{\cos^2(\omega)}{\lambda+n_1}- \frac{\sin(\omega)(1-\sin(\omega))}{\lambda+n_2}}}}{\sqrt{1+\frac{\cos^2(\omega)}{\lambda+n_1}+\frac{\sin^2(\omega)}{\lambda+n_2}}}=& \frac{-\frac{\cos^2(\omega)}{\lambda+n_1}+ \frac{\sin(\omega)(1-\sin(\omega))}{\lambda+n_2}}{\sqrt{1+\frac{\cos^2(\omega)}{\lambda+n_1}+\frac{\sin^2(\omega)}{\lambda+n_2}}}\\
<& \frac{\frac{\sin(\omega)(1-\sin(\omega))}{\lambda+n_2}}{\sqrt{1+\frac{\cos^2(\omega)}{\lambda+n_1}+\frac{\sin^2(\omega)}{\lambda+n_2}}}\\
<& \frac{\frac{\sin(\omega)(1-\sin(\omega))}{\lambda+n_2}}{\sqrt{1+\frac{\sin^2(\omega)}{\lambda+n_2}}}\\
=&\frac{\sin(\omega)(1-\sin(\omega))}{\sqrt{(\lambda+n_2)(\lambda+n_2+\sin^2(\omega))}}\\
=&(1-\sin(\omega))\Bigg(\frac{\sin(\omega)}{\sqrt{(\lambda+n_2)(\lambda+n_2+\sin^2(\omega))}}\Bigg)\\
\leq&  (1-\sin(\omega))\frac{1}{\sqrt{(\lambda+n_2)(\lambda+n_2+1)}}
\end{align*}
where, again, the last inequality follows by noting that $\frac{x}{\sqrt{a+x^2}}$ is increasing in $x\geq0$ for $a>0.$
\par We just showed that arm 3 is not played even in this sub-case.

\par \textit{Sub-case 3: $\{h_*^t,l_*^t\}\in \{1,3\}$.} In this case the algorithm plays:
\[\argmax\Bigg\{ \frac{1-\cos(\omega)}{\sqrt{(\lambda+n_1)(\lambda+n_1+1)}}, \frac{\sin(\omega)}{\sqrt{(\lambda+n_2)(\lambda+n_2+1)}} , \frac{\abs{\frac{\cos(\omega)(1-\cos(\omega))}{\lambda+n_1}-\frac{\sin^2(\omega)}{\lambda+n_2}}}{\sqrt{1+\frac{\cos^2(\omega)}{\lambda+n_1}+\frac{\sin^2(\omega)}{\lambda+n_2}}}         \Bigg \} \]
Again, we have 2 cases depending on the sign of the last term. \\
\underline{If $ {\frac{\cos(\omega)(1-\cos(\omega))}{\lambda+n_1}\geq \frac{\sin^2(\omega)}{\lambda+n_2}} $}: The third term  becomes:
\begin{align*}
\frac{\abs{\frac{\cos(\omega)(1-\cos(\omega))}{\lambda+n_1}- \frac{\sin^2(\omega)}{\lambda+n_2}}}{\sqrt{1+\frac{\cos^2(\omega)}{\lambda+n_1}+\frac{\sin^2(\omega)}{\lambda+n_2}}}=& \frac{\frac{\cos(\omega)(1-\cos(\omega))}{\lambda+n_1}- \frac{\sin^2(\omega)}{\lambda+n_2}}{\sqrt{1+\frac{\cos^2(\omega)}{\lambda+n_1}+\frac{\sin^2(\omega)}{\lambda+n_2}}}\\
\leq& \frac{{\frac{\cos(\omega)(1-\cos(\omega))}{\lambda+n_1}}}{\sqrt{1+\frac{\cos^2(\omega)}{\lambda+n_1}+\frac{\sin^2(\omega)}{\lambda+n_2}}}\\
<& \frac{{\frac{\cos(\omega)(1-\cos(\omega))}{\lambda+n_1}}}{\sqrt{1+\frac{\cos^2(\omega)}{\lambda+n_1}}}\\
=&\frac{\cos(\omega)(1-\cos(\omega))}{\sqrt{(\lambda+n_1)(\lambda+n_1+\cos^2(\omega))}}\\
=&(1-\cos(\omega))\Bigg(\frac{\cos(\omega)}{\sqrt{(\lambda+n_1)(\lambda+n_1+\cos^2(\omega))}}\Bigg)\\
\leq&  (1-\cos(\omega))\frac{1}{\sqrt{(\lambda+n_1)(\lambda+n_1+1)}}.
\end{align*}

\par \underline{Else if $ {\frac{\cos(\omega)(1-\cos(\omega))}{\lambda+n_1} < \frac{\sin^2(\omega)}{\lambda+n_2}} $}: The term  becomes:
\begin{align*}
\frac{\abs{\frac{\cos(\omega)(1-\cos(\omega))}{\lambda+n_1} - \frac{\sin^2(\omega)}{\lambda+n_2}}}{\sqrt{1+\frac{\cos^2(\omega)}{\lambda+n_1}+\frac{\sin^2(\omega)}{\lambda+n_2}}}=& \frac{\frac{\sin^2(\omega)}{\lambda+n_2}- \frac{\cos(\omega)(1-\cos(\omega))}{\lambda+n_1}}{\sqrt{1+\frac{\cos^2(\omega)}{\lambda+n_1}+\frac{\sin^2(\omega)}{\lambda+n_2}}}\\
<& \frac{\frac{\sin^2(\omega)}{\lambda+n_2}}{\sqrt{1+\frac{\cos^2(\omega)}{\lambda+n_1}+\frac{\sin^2(\omega)}{\lambda+n_2}}}\\
<& \frac{\frac{\sin^2(\omega)}{\lambda+n_2}}{\sqrt{1+\frac{\sin^2(\omega)}{\lambda+n_2}}}\\
=&\frac{\sin^2(\omega)}{\sqrt{(\lambda+n_2)(\lambda+n_2+\sin^2(\omega))}}\\
=&(\sin(\omega))\Bigg(\frac{\sin(\omega)}{\sqrt{(\lambda+n_2)(\lambda+n_2+\sin^2(\omega))}}\Bigg)\\
\leq&  (\sin(\omega))\frac{1}{\sqrt{(\lambda+n_2)(\lambda+n_2+1)}}
\end{align*}

\par We just proved that arm 3 will \textit{never} be played under the condition that $0<\omega<\pi/2.$

\end{proof}

\begin{proof}[Proof of lemma \ref{lemma:3arm1stlemma}]\label{proof of lemma 3 arm case}
At any time $t$, we note from the mechanics of the proof of Lemma \ref{lemma:key3armcase}, that GLUCB follows one of the following three rules, depending on the current realization of the set $\{h_*^t,l_*^t \}$.
\begin{enumerate}
\item[If $ \{h_*^t,l_*^t \}\equiv \{1,2 \} $]: Play arm 1 if 
\begin{align*}
\frac{1}{\sqrt{(\lambda+n_1)(\lambda+n_1+1)}} \geq \frac{1}{\sqrt{(\lambda+n_2)(\lambda+n_2+1)}}\\
\equiv n_2\geq n_1,
\end{align*}
else play Arm 2. Hence GLUCB follows the following rule
\begin{equation}\label{eq:12}
a_{t+1}=
\begin{cases}
1 & \text{if } n_2\geq n_1\\
2 & \text{otherwise}.
\end{cases}
\end{equation}
We see a \textit{self-correcting } property of GLUCB which tries to stabilize the pull counts of the arms according to some ratio.
\item[If $ \{h_*^t,l_*^t \}\equiv \{2,3 \} $] Play arm 1 if 
\begin{align*}
\frac{\cos(\omega)}{\sqrt{(\lambda+n_1)(\lambda+n_1+1)}} \geq \frac{1-\sin(\omega)}{\sqrt{(\lambda+n_2)(\lambda+n_2+1)}}\\
\Leftrightarrow \sqrt{\frac{(\lambda+n_2)(\lambda+n_2+1)}{(\lambda+n_1)(\lambda+n_1+1)}} \geq \frac{1-\sin(\omega)}{\cos(\omega)}
\approx n_2\geq n_1 \frac{1-\sin(\omega)}{\cos(\omega)}.
\end{align*}
else play Arm 2. Hence GLUCB (approximately) follows the following rule
\begin{equation}\label{eq:23}
a_{t+1}=
\begin{cases}
1 &  \text{if } n_2\geq n_1 \frac{1-\sin(\omega)}{\cos(\omega)}\\
2 & \text{otherwise}.
\end{cases}
\end{equation}
\item[If $ \{h_*^t,l_*^t \}\equiv \{1,3 \} $]Play arm 1 if 
\begin{align*}
\frac{1-\cos(\omega)}{\sqrt{(\lambda+n_1)(\lambda+n_1+1)}} \geq \frac{\sin(\omega)}{\sqrt{(\lambda+n_2)(\lambda+n_2+1)}}\\
\Leftrightarrow \sqrt{\frac{(\lambda+n_2)(\lambda+n_2+1)}{(\lambda+n_1)(\lambda+n_1+1)}} \geq \frac{\sin(\omega)}{1-\cos(\omega)}
\approx n_2\geq n_1 \frac{\sin(\omega)}{1-\cos(\omega)}.
\end{align*}
else play Arm 2. Hence again, GLUCB (approximately) follows the following rule
\begin{equation}\label{eq:13}
a_{t+1}=
\begin{cases}
1 &  \text{if } n_2\geq n_1 \frac{\sin(\omega)}{1-\cos(\omega)}\\
2 & \text{otherwise}.
\end{cases}
\end{equation}
\end{enumerate}

In all of the three cases above we see a \textit{self-correcting property} of GLUCB, which means that the number of plays of arms 1 and 2 can never be very different from each other. They are tied together according to a relation at each round precisely given by equations \ref{eq:12}, \ref{eq:23} and \ref{eq:13}. Hence we can write that there exists finite constants $a>0$ and $b>0$ such that at any time $t$,  
\begin{equation}\label{eq:attractive property for 3 arms}
an_2(t)\leq n_1(t) \leq bn_2(t).
\end{equation}
We are now ready to prove the lemma. We note that the claim will be implied by bounding the number of samples required by GLUCB to make $\text{Advantage}(2)<0$. By definition of Advantage, at any time $t$,
\[\text{Advantage}(2)=\theta_t^(x_2-x_{h_*^t})+\beta_t\norm{x_2-x_{h_*^t}}_{V_t^{-1}},\]
where, $h_*^t=1$ or $2$. We compute the number of samples required for each case and take the max.  
\par If $x_{h_*^t}\equiv 1$, we have under the good event $\mathcal{E}$:

\begin{align*}
&\theta_t^T(x_2-x_1)+\beta_t\norm{x_2-x_1}_{V_t^{-1}}<0\\
&\Leftarrow {\theta^*}^(e_2-e_1)+2\beta_t\norm{e_2-e_1}_{V_t^{-1}}<0\\
&\Leftrightarrow 2\beta_t\norm{e_2-e_1}_{V_t^{-1}}< \Delta_2\\
&\Leftarrow \frac{1}{\lambda+n_1} +\frac{1}{\lambda+n_2}< \frac{\Delta_2^2}{4\beta_t^2}\\
&\Leftarrow \frac{1}{\lambda+an_2} +\frac{1}{\lambda+n_2}< \frac{\Delta_2^2}{4\beta_t^2}\\
&\Leftarrow \frac{1}{an_2} +\frac{1}{n_2}<\frac{\Delta_2^2}{4\beta_t^2}\\
&\Leftrightarrow n_2 > \frac{4\beta_t^2}{\Delta_2^2}\frac{1+a}{a}\\
& \Leftrightarrow n_2 > {4\beta_t^2}\frac{1+a}{a},
\end{align*}
as $\Delta_2=1$. Next, if If $x_{h_*^t}\equiv 3$, we have under the good event $\mathcal{E}$:

\begin{align*}
&\theta_t^T(x_2-x_3)+\beta_t\norm{x_2-x_3}_{V_t^{-1}}<0\\
&\Leftarrow {\theta^*}^(e_2-x_3)+2\beta_t\norm{e_2-x_3}_{V_t^{-1}}<0\\
&\Leftrightarrow 2\beta_t\norm{e_2-x_3}_{V_t^{-1}}< \mu_3 \\
&\Leftarrow \frac{\cos(\omega)}{\lambda+n_1} +\frac{1-\sin(\omega)}{\lambda+n_2}< \frac{\mu_3^2}{4\beta_t^2}\\
&\Leftarrow \frac{\cos(\omega)}{\lambda+an_2} +\frac{1-\sin(\omega)}{\lambda+n_2}< \frac{\mu_3^2}{4\beta_t^2}\\
&\Leftarrow \frac{\cos(\omega)}{an_2} +\frac{1-\sin(\omega)}{n_2}<\frac{\mu_3^2}{4\beta_t^2}\\
&\Leftrightarrow n_2 > \frac{4\beta_t^2}{\mu_3^2}\left(\frac{\cos(\omega)}{a} + 1-\sin(\omega)\right).
\end{align*}
Hence, after $t'=\max\left\{{4\beta_t^2}\frac{1+a}{a},\frac{4\beta_t^2}{\mu_3^2}\left(\frac{\cos(\omega)}{a} + 1-\sin(\omega)\right) \right\}$ rounds, with high probability, GLUCB discards Arm 2 from the set $\{h_*^t,l_*^t \}$  for all $t\geq t'+1.$ We observe that foe small values of $\omega$, $a:=\min\left\{1, \frac{\cos(\omega)}{1-sin(\omega)}, \frac{1-\cos(\omega)}{\sin(\omega)} \right\} \approx \min\left\{1, \frac{1}{1-\omega}, \frac{1-\cos(\omega)}{\sin(\omega)} \right\} = \min\left\{1, \frac{1}{1-\omega}, \frac{\Delta_{min}}{\sin(\omega)} \right\}=\frac{\Delta_{min}}{\sin(\omega)}$ . Putting this value of $a$, gives the required result.

\end{proof}
\subsubsection{Proof of lemma \ref{lemma:3arm2ndlemma}}
\begin{proof}
Recall that steady state was defined as the time after which the set $\{h_*^t, l_*^t \}$ freezes as $\{1,3\}$. Let us calculate the relation between $n_1$ and $n_2$ at any time $t$ in the steady state. For this we note that we will play arm 1 if:
\begin{align*}
\frac{\abs{e_1^TV_t^{-1}(e_1-x_3)}}{\sqrt{1+e_1^TV_t^{-1}e_1}} \geq& \frac{\abs{e_2^TV_t^{-1}(e_1-x_3)}}{\sqrt{1+e_2^TV_t^{-1}e_2}}\\
\Leftrightarrow \frac{\abs{\frac{1-\cos(\omega)}{\lambda+ n_1}}}{\sqrt{1+\frac{1}{\lambda+n_1}}} \geq& \frac{\abs{\frac{\sin(\omega)}{\lambda+ n_2}}}{\sqrt{1+\frac{1}{\lambda+n_2}}}\\
\Leftrightarrow \frac{\abs{1-\cos(\omega)}}{\sqrt{(\lambda+n_1)(1+\lambda+n_1)}} \geq&\frac{\abs{\sin(\omega)}}{\sqrt{(\lambda+n_2)(1+\lambda+n_2)}}\\ 
\end{align*} 
For large values of $n_1$ and $n_2$, we make the following approximation:
\begin{align*}
\approx \frac{n_2}{n_1}\geq& \frac{\abs{\sin(\omega)}}{\abs{1-\cos(\omega)}}
\\
=&\frac{{\sin(\omega)}}{{1-\cos(\omega)}}=\frac{\sin(\omega)}{\Delta_{min}}\\
\Leftrightarrow n_1 \leq& n_2\frac{\Delta_{min}}{\sin(\omega)}.
\end{align*}
Hence we imply the following: In steady-state, G-LUCB keeps on playing arm 2, till the above inequality is reached, when it switches and plays arm 1 and continues to do so until the inequality is reversed, and repeats. 
\par From the stopping criteria, the algorithm stops at some time $t$ if:
\[\tilde{\theta}_t^T(x_{l_*^t}-x_{h_*^t}) +\beta_t\norm{x_{l_*^t}-x_{h_*^t}}_{V_{t}^{-1}} <0. \]
From the definition of steady state and under the good event $\mathcal{E}$, $h_*^t\equiv 1$ and $l_*^t\equiv 3$. 
\begin{align*}
\tilde{\theta}_t^T(x_{3}-x_{1}) +\beta_t\norm{x_{3}-x_{1}}_{V_{t}^{-1}} &<0\\
\Leftarrow {\theta^*}^T(x_{3}-x_{1}) + +2\beta_t\norm{x_{3}-x_{1}}_{V_{t}^{-1}}&<0\\
\Leftrightarrow 2\beta_t\norm{x_{3}-x_{1}}_{V_{t}^{-1}}&< \Delta_{min}\\
\Leftrightarrow \norm{x_{3}-x_{1}}^2_{V_{t}^{-1}}&< \frac{\Delta_{min}^2}{4\beta_t^2}\\
\Leftrightarrow \begin{bmatrix*}
1-\cos(\omega) \quad \sin(\omega)
\end{bmatrix*}
\begin{bmatrix*}
\frac{1}{\lambda+N_1} \quad 0\\
0\quad \frac{1}{\lambda+N_2}
\end{bmatrix*}
\begin{bmatrix*}
1-\cos(\omega) \quad \sin(\omega)
\end{bmatrix*}&<\frac{\Delta_{min}^2}{4\beta_t^2}\\
\Leftrightarrow \frac{(1-\cos(\omega))^2}{\lambda +N_1} + \frac{(\sin(\omega))^2}{\lambda +N_2}&<\frac{\Delta_{min}^2}{4\beta_t^2}\\
\Leftrightarrow \frac{\Delta_{min}^2}{\lambda +N_1} + \frac{\sin^2(\omega)}{\lambda +N_2}&<\frac{\Delta_{min}^2}{4\beta_t^2}\\
\Leftarrow \frac{\Delta_{min}^2}{N_1} + \frac{\sin^2(\omega)}{N_2}&<\frac{\Delta_{min}^2}{4\beta_t^2}\\
\Leftrightarrow \frac{\Delta_{min}^2}{\frac{\Delta_{min}}{\sin(\omega)}N_2} + \frac{\sin^2(\omega)}{N_2}&<\frac{\Delta_{min}^2}{4\beta_t^2}\\
\Leftrightarrow \frac{\Delta_{min}\sin(\omega)+ \sin^2(\omega)}{N_2}&<\frac{\Delta_{min}^2}{4\beta_t^2}\\
\Leftrightarrow N_2 &> \frac{4\beta_t^2}{\Delta_{min}^2}\Bigg(\Delta_{min}\sin(\omega)+ \sin^2(\omega)  \Bigg)\\
 &= \frac{4\beta_t^2}{\Delta_{min}}\sin(\omega) +\frac{4\beta_t^2}{\Delta_{min}^2}\sin^2(\omega)\\
&=O\Bigg(\frac{4\beta_t^2}{\Delta_{min}^2}{\sin^2(\omega)}\Bigg).
\end{align*}

\end{proof}
\end{document}